\documentclass[twoside]{article}
\usepackage[utf8]{inputenc} 
\usepackage[T1]{fontenc}    
\usepackage{hyperref}       
\usepackage{url}            
\usepackage{booktabs}       
\usepackage{amsfonts}       
\usepackage{nicefrac}       
\usepackage{microtype}      
\usepackage{color}
\usepackage{times}
\usepackage{graphicx} 
\usepackage{subcaption}
\usepackage{algorithm}
\usepackage{algorithmic}
\usepackage{hyperref}
\usepackage{paralist}
\usepackage{multirow}
\usepackage{amsmath,amssymb,amsthm,mathrsfs,amsfonts,dsfont}
\usepackage[textsize=tiny]{todonotes}

\usepackage{prettyref}
\usepackage{thmtools,thm-restate}

\usepackage[accepted]{aistats2018}
\newcommand\ind[1]{\ensuremath{\mathds{1}\left[#1\right]}}
\newcommand{\setfont}[1]{\mathbb{#1}}
\def\argmin{\ensuremath{\mbox{argmin}}}

\newcommand{\N}{\mathcal{N}}
\newcommand{\E}[2]{\mathbf{E}_{#1}\left[#2\right]}

\renewcommand{\P}[2]{\mathbf{P}_{#1}\left[#2\right]}
\newcommand{\R}{\setfont{R}}

\newcommand{\V}[1]{\mathbf{V}\left[#1\right]}
\newcommand{\Cov}[2]{\mathbf{Cov}\left[#1,#2\right]}
\newcommand{\X}{\setfont{X}}
\newcommand{\Y}{\setfont{Y}}
\newcommand{\Z}{\setfont{Z}}
\newcommand{\M}{\setfont{M}}

\newcommand{\defeq}{\triangleq}

\newcommand{\thmref}[1]{Theorem~\ref{#1}}
\newcommand{\lemref}[1]{Lemma~\ref{#1}}
\newcommand{\propref}[1]{Proposition~\ref{#1}}

\newtheorem{theorem}{Theorem}
\newtheorem{lemma}[theorem]{Lemma}

\newtheorem{definition}{Definition}
\newtheorem{example}{Example}

\begin{document}

%

%

\twocolumn[

\aistatstitle{Teacher Improves Learning by Selecting a Training Subset}

\aistatsauthor{ Yuzhe Ma \And Robert Nowak \And  Philippe Rigollet$^{*}$ \And Xuezhou Zhang \And Xiaojin Zhu}

\aistatsaddress{ University of Wisconsin-Madison \And  $^*$Massachusetts Institute of Technology} ]

\begin{abstract}
We call a learner super-teachable if a teacher can trim down an $iid$ training set while making the learner learn even better.
We provide sharp super-teaching guarantees on two learners: the maximum likelihood estimator for the mean of a Gaussian, and the large margin classifier in 1D.
For general learners, we provide a mixed-integer nonlinear programming-based algorithm to find a super teaching set. Empirical experiments show that our algorithm is able to find good super-teaching sets for both regression and classification problems.
\end{abstract}

\section{Introduction}
Consider the following question: a learner receives an $iid$ training set $S$ drawn from a distribution parametrized by $\theta^*$.
There is a teacher who knows $\theta^*$.
Can the teacher select a subset from $S$ so the learner estimates $\theta^*$ better from the subset than from $S$?

This question is distinct from training set reduction (see e.g.~\cite{garcia2012prototype,zeng2005smo,Wilson2000}) in that the teacher can use the knowledge of $\theta^*$ to carefully design the subset.
It is, in fact, a coding problem: Can the teacher approximately encode $\theta^*$ using items in $S$ for a known decoder, which is the learner?
As such, the question is not a machine learning task but rather a machine teaching one~\cite{Zhu2018Overview,Goldman1995Complexity,Zhu2015Machine}.

This question is relevant for several nascent applications.
One application is in understanding blackbox models such as deep nets.
Often observation to a blackbox model is limited to its predicted label $y=\theta^*(x)$ given input $x$.
One way to interpret a blackbox model is to locally train an interpretable model with data points $S$ labeled by the blackbox model around the region of interest~\cite{ribeiro2016should}.
We, however, ask for more: to reduce the size of the training set $S$ for the local learner \emph{while} making the learner approximate the blackbox better. The reduced training set itself also serves as representative examples of local model behavior.
Another application is in education.
Imagine a teacher who has a teaching goal $\theta^*$.
This is a reasonable assumption in practice: e.g. a geology teacher has the knowledge of the actual decision boundaries between rock categories.
However, the teacher is constrained to teach with a given textbook (or a set of courseware) $S$.
To the extent that the student is quantified mathematically, the teacher wants to select pages in the textbook with the guarantee that the student learns better from those pages than from gulping the whole book.

But is the question possible?
The following example says yes.
Consider learning a threshold classifier on the interval $[-1,1]$, with true threshold at $\theta^*=0$.
Let $S$ have $n$ items drawn uniformly from the interval and labeled according to $\theta^*$.
Let the learner be a hard margin SVM, which 
places the estimated threshold in the middle of the inner-most pair in $S$ with different labels:
$\hat\theta_S = (x_- + x_+)/2$
where $x_-$ is the largest negative training item and $x_+$ the smallest positive training item in $S$.
It is well known that $|\hat\theta_S-\theta^*|$ converges at a rate of $1/n$: the intuition being that the average space between adjacent items is $O(1/n)$.
\begin{figure}[h]
\centerline{\includegraphics[width=.5\textwidth]{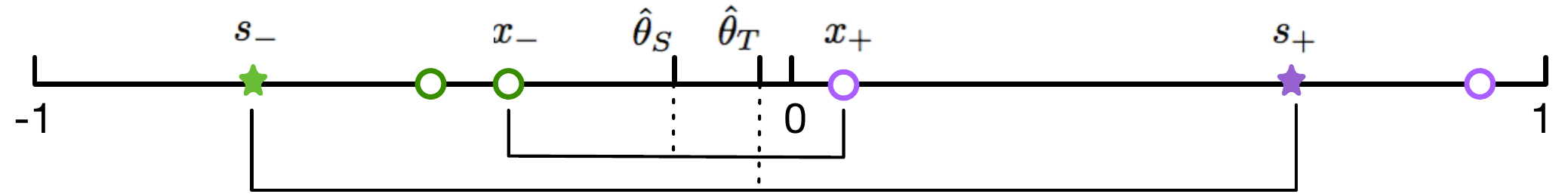}}
\caption{The original training set $S$ with $n=6$ items (circles and stars; green=negative, purple=positive), and the most-symmetric training set (stars) the teacher selects.} 
\label{fig:symmetric}
\end{figure}

The teacher knows everything but cannot tell $\theta^*$ directly to the learner.
Instead, it can select the \emph{most-symmetric pair} in $S$ about $\theta^*$ and give them to the learner as a two-item training set.
We will prove later that the risk on the most symmetric pair is $O(1/n^2)$, that is, learning from the selected subset surpasses learning from $S$.
Thus we observe something interesting: the teacher can turn a larger training set $S$ into a smaller and better subset for the midpoint classifier.
We call this phenomenon \textbf{super-teaching}.

\section{Formal Definition of Super Teaching}
\label{sec:def}
Let $\Z$ be the data space: for unsupervised learning $\Z=\X$, while for supervised learning $\Z=\X \times \Y$.
Let $p_\Z$ be the underlying distribution over $\Z$.
We take a function view of the learner: a learner $A$ is a function $A: \cup_{n=0}^\infty \Z^n \mapsto \Theta$, where $\Theta$ is the learner's hypothesis space.
The notation $\cup_{n=0}^\infty \Z^n$ defines the ``set of (potentially non-$iid$) training sets'',
namely multisets of any size whose elements are in $\Z$.
Given any training set $T \in \cup_{n=0}^\infty \Z^n$, we assume $A$ returns a unique hypothesis $A(T) \defeq \hat\theta_T \in \Theta$.
The learner's risk $R(\theta)$ for $\theta\in\Theta$ is defined as:
\begin{equation}
R(\theta)=\E{p_\Z}{\ell(\theta(x), y)}, \mbox{ or }
R(\theta)= \|\theta-\theta^*\|_2.
\label{eq:R}
\end{equation}
The former is for prediction tasks where $\ell()$ is a loss function and $\theta(x)$ denotes the prediction on $x$ made by model $\theta$;
the latter is for parameter estimation where we assume a realizable model $p_\Z = p_{\theta^*}$ for some $\theta^* \in \Theta$.

We now introduce a clairvoyant teacher $B$ who has full knowledge of $p_\Z, A, R$.
The teacher is also given an $iid$ training set $S=\{z_1, \ldots, z_n\} \sim p_\Z$. 
If the teacher teaches $S$ to $A$, the learner will incur a risk $R(A(S)) \defeq R(\hat\theta_S)$.
The teacher's goal is to judiciously select a subset $B(S) \subset S$ to act as a ``super teaching set'' for the learner so that $R(\hat\theta_{B(S)}) < R(\hat\theta_S)$.  
Of course, to do so the teacher must utilize her knowledge of the learning task, thus the subset is actually a function $B(S, p_\Z, A, R)$.  
In particular, the teacher knows $p_\Z$ already, and this sets our problem apart from machine learning.
For readability we suppress these extra parameters in the rest of the paper.
We formally define super teaching as follows.

\begin{definition}[Super Teaching]
\label{def:superteaching}
$B$ is a super teacher for learner $A$ if~$\forall\delta>0, \exists N$ such that $\forall n\ge N$
\begin{equation}
\P{S}{R(\hat\theta_{B(S)}) \le c_n R(\hat\theta_S)} > 1-\delta,
\end{equation}
where 
$S \stackrel{iid}{\sim} p_\Z^n, B(S)\subset S$, and $c_n \le 1$ is a sequence we call super teaching ratio.
\end{definition}
Obviously, $c_n=1$ can be trivially achieved by letting $B(S)=S$ so we are interested in small $c_n$.
There are two fundamental questions:
(1) Do super teachers provably exist?
(2) How to compute a super teaching set $B(S)$ in practice?

We answer the first question positively by exhibiting super teaching on two learners:
maximum likelihood estimator for the mean of a Gaussian in section~\ref{sec:Gaussian},
and
1D large margin classifier in section~\ref{sec:midpoint}. 
Guarantees on super teaching for general learners remain future work.
Nonetheless, empirically we can find a super teaching set for many general learners:
We formulate the second question as mixed-integer nonlinear programming in section~\ref{sec:MINLP}.
Empirical experiments in section~\ref{sec:exp} demonstrates that one can find a good $B(S)$ effectively.


\section{Analysis on Super Teaching for the MLE of Gaussian mean}
\label{sec:Gaussian}
In this section, we present our first theoretical result on super teaching, when the learner $A_{MLE}$ is the maximum likelihood estimator (MLE) for the mean of a Gaussian. 
Let $\Z=\X=\R$, $\Theta=\R$, $p_\Z(x)=\N(\theta^*, 1)$.
Given a sample $S$ of size $n$ drawn from $p_\Z$,
the learner computes the MLE for the mean: $\hat \theta_S =A_{MLE}(S)= \frac{1}{n}\sum_{i=1}^n x_i$.
We define the risk as $R(\hat \theta_S)=|\hat \theta_S-\theta^*|$.
The teacher we consider is the optimal $k$-subset teacher $B_k$, which uses the best subset of size $k$ to teach:
\begin{equation}
B_k(S) \in \argmin_{T \subset S, |T|=k} R(\hat\theta_T).
\end{equation}

To build intuition, it is well-known that the risk of $A_{MLE}$ under $S$ is $O(1/\sqrt{n})$ because the variance under $n$ items shrinks like $1/n$.
Now consider $k=1$.
Since the teacher $B_1$ knows $\theta^*$, under our setting the best teaching strategy is for her to select the item in $S$ closest to $\theta^*$, which forms the singleton teaching set $B_1(S)$.
One can show that with large probability this closest item is $O(1/n)$ away from $\theta^*$
(the central part of a Gaussian density is essentially uniform).
Therefore, we already see a super teaching ratio of $c_n = n^{-\frac{1}{2}}$.
More generally, our main result below shows that $B_k$ achieves a super teaching ratio $c_n = O(n^{-k+\frac{1}{2}})$:
\begin{restatable}{theorem}{Amuthm}
\label{thm:Amuthm}
Let $B_k$ be the optimal $k$-subset teacher.
$\forall \epsilon\in(0,\frac{2k-1}{4}), \forall\delta\in(0, 1)$, $\exists N(k, \epsilon, \delta)$ such that $\forall n\ge N(k, \epsilon, \delta)$, $\P{}{R(\hat \theta_{B_k(S)})\le c_nR(\hat \theta_S)}>1-\delta$, where $c_n=\frac{k^{k-\epsilon}}{\sqrt{k}}n^{-k+\frac{1}{2}+2\epsilon}$.
\end{restatable}
Toward proving the theorem,
\footnote{Remark: we introduced an auxiliary variable $\epsilon$ which controls the implicit tradeoff between $c_n$, how much super teaching helps, and $N$, how soon super teaching takes effect.
When $\epsilon\rightarrow 0$ the teaching ratio $c_n$ approaches $O(n^{-k+\frac{1}{2}})$, but as we will see $N(k, \epsilon, \delta)\rightarrow\infty$.
Similarly, $k$ also affects the tradeoff: the teaching ratio is smaller as we enlarge $k$, but 
$N(k, \epsilon, \delta)$ increases.}
we first recall the standard rate $R(\hat\theta_S) \approx n^{-\frac{1}{2}}$ if $A_{MLE}$ learns from the whole training set $S$:

\begin{restatable}{proposition}{thmAsyGaussianPool}
\label{thm:AsyGaussianPool}
Let $S$ be an $n$-item $iid$ sample drawn from $ \N(\theta^*, 1)$. $\forall \epsilon>0$, $\forall\delta\in(0,1)$, $\exists N_1(\epsilon, \delta)$ such that $\forall n\ge N_1$, 
\begin{equation}
\P{}{n^{-\frac{1}{2}-\epsilon}<R(\hat \theta_S)<n^{-\frac{1}{2}+\epsilon}}>1-\delta.
\end{equation}
\end{restatable}
\begin{proof}
$R(\hat \theta_S)=|\hat \theta_S-\theta^*|$ and $\hat \theta_S-\theta^*\sim \N(0, n^{-1})=\sqrt{\frac{n}{2\pi}}e^{-\frac{nx^2}{2}}$. Let $\alpha=n^{-\frac{1}{2}-\epsilon}$ and $\beta=n^{-\frac{1}{2}+\epsilon}$. 
We have
\begin{equation}
\begin{aligned}
\label{lower}
&\P{}{R(\hat\theta_S)\le\alpha}=2\int_{0}^\alpha \sqrt{\frac{n}{2\pi}}e^{-\frac{nx^2}{2}}dx\\
&<2\int_{0}^\alpha \sqrt{\frac{n}{2\pi}}dx=2\alpha\sqrt{\frac{n}{2\pi}}=\sqrt{\frac{2}{\pi}}n^{-\epsilon},
\end{aligned}
\end{equation}
\begin{equation}
\label{upper}
\begin{aligned}
&\P{}{R(\hat\theta_S)\ge\beta}=2\int_{\beta}^\infty \sqrt{\frac{n}{2\pi}}e^{-\frac{nx^2}{2}}dx\\
&<2\int_{\beta}^\infty \frac{x}{\beta}\sqrt{\frac{n}{2\pi}}e^{-\frac{nx^2}{2}}dx=\int_{\beta^2}^\infty \frac{1}{\beta}\sqrt{\frac{n}{2\pi}}e^{-\frac{ny}{2}}dy\\
&=\frac{1}{\beta}\sqrt{\frac{2}{n\pi}}e^{-\frac{n\beta^2}{2}}<\frac{1}{\beta}\sqrt{\frac{2}{n\pi}}=\sqrt{\frac{2}{\pi}}n^{-\epsilon}.
\end{aligned}
\end{equation}
Thus 
$\P{}{\alpha< R(\hat\theta_S)<\beta}
=1-\P{}{R(\hat\theta_S)\le \alpha}-\P{}{R(\hat\theta_S)\ge \beta}
>1-2\sqrt{\frac{2}{\pi}}n^{-\epsilon}$. 
Let $N_1(\epsilon, \delta)=(\frac{1}{\delta}\sqrt{\frac{8}{\pi}})^{\frac{1}{\epsilon}}$, then $\forall n\ge N_1$, $\P{}{\alpha< R(\hat\theta_S)<\beta}>1-\delta$.
\end{proof}
We now work out the risk of $A_{MLE}$ if it learns from the optimal $k$-subset teacher $B_k$. 
\thmref{thm:Gaussian} says that this risk is very small and sharply concentrated around $R(\hat\theta_{B_k(S)}) \approx n^{-k}$. To prove~\thmref{thm:Gaussian}, we first give the following lemma.
\begin{lemma}
Denote $C^n_k=\begin{pmatrix} n\\k\end{pmatrix}$. Let the index set $I=\{1,2,...n\}$ where $n\ge 4k$. Consider all subsets of size $k$, then there are at most $4^kC^{2k}_kC^n_{2k-1}$ ordered pairs of subsets that are overlapping but not identical.
\label{lem:intersect}
\end{lemma}
\begin{proof}
Let $I_1$ and $I_2$ be two subsets of size $k$ and they overlap on $t$ indexes. Then the total number of distinct indexes that appear in $I_1\cup I_2$ is $2k-t$. There are $C^n_{2k-t}$ ways of choosing such $2k-t$ indexes. Next we determine which $t$ indexes are overlapping ones. We have $C^{2k-t}_t$ ways of choosing such $t$ indexes. Finally we have $C^{2k-2t}_{k-t}$ ways of selecting half of the non-overlapping indexes and attribute them to $I_1$. Thus in total we have $O_t=C^n_{2k-t}C^{2k-t}_tC^{2k-2t}_{k-t}$ ordered pairs of subsets that overlap on $t$ indexes. By our assumption $n\ge4k$ we have $C^n_{2k-t}\le C^n_{2k-1}$. Also note that $C^{2k-t}_t<C^{2k}_t$ and $C^{2k-2t}_{k-t}<C^{2k}_k$, thus $O_t<C^n_{2k-1}C^{2k}_tC^{2k}_k$. Therefore the total number of ordered pairs of subsets that are overlapping but not identical is
\begin{equation}
\begin{aligned}
&O=\sum_{t=1}^{k-1}O_t<\sum_{t=1}^{k-1}C^n_{2k-1}C^{2k}_tC^{2k}_k\\
&<\sum_{t=0}^{2k}C^n_{2k-1}C^{2k}_tC^{2k}_k=4^kC^{2k}_kC^n_{2k-1}.
\end{aligned}
\end{equation}
\end{proof}
Now we prove the risk of the optimal $k$-subset teacher.
\begin{restatable}{theorem}{thmAsyGaussian}
\label{thm:Gaussian}
Let $B_k$ be the optimal $k$-subset teacher. Let $S$ be an $n$-item $iid$ sample drawn from $\N(\theta^*, 1)$. $\forall \epsilon\in(0, k), \forall \delta\in(0,1)$, $\exists N_2(k, \epsilon, \delta)$ such that $ \forall n\ge N_2$,
\begin{equation}
\P{}{\frac{1}{\sqrt{k}}(\frac{k}{n})^{k+\epsilon}<R(\hat \theta_{B_k(S)})<\frac{1}{\sqrt{k}}(\frac{k}{n})^{k-\epsilon}}>1-\delta.
\end{equation}
\end{restatable}
\begin{proof}
Let $I\subseteq\{1,2,...,n\}$ and $|I|=k$, define $\gamma_I=\frac{1}{\sqrt{k}}\sum_{i\in I}(x_i-\theta^*)$. Let $S_I$ denote the subset indexed by $I$.
Note that 
$\hat\theta_{S_I}=\frac{1}{k}\sum_{i\in I}x_i$ and
$R(\hat\theta_{S_I})=|\hat\theta_{S_I}-\theta^*|=|\frac{1}{k}\sum_{i\in I}x_i-\theta^*|=\frac{1}{\sqrt{k}}|\gamma_I|$. 
Also note that $R(\hat \theta_{B_k(S)})=\inf_{I}R(\hat\theta_{S_I})=\frac{1}{\sqrt{k}}\inf_{I}|\gamma_I|$. 
Thus to prove \thmref{thm:Gaussian} it suffices to prove 
\begin{equation}
\P{}{(\frac{k}{n})^{k+\epsilon}<\inf_{I}|\gamma_I|<(\frac{k}{n})^{k-\epsilon}}\rightarrow1.
\end{equation}
Let $\alpha=(\frac{k}{n})^{k+\epsilon}$ and $\beta=(\frac{k}{n})^{k-\epsilon}$. We first prove the lower bound. Note that $\gamma_I$ has the same distribution for all $I$. Thus by the union bound,
\begin{equation}
\begin{aligned}
\P{}{\inf_{I}|\gamma_I|\le \alpha}=\P{}{\exists I: |\gamma_I|\le\alpha}\le C^n_k\P{}{|\gamma_{I_1}|\le\alpha},
\end{aligned}
\end{equation}
where $I_1=\{1,2,...,k\}$. Since $\gamma_{I_1}\sim \N(0, 1)$, we have
\begin{equation}
\P{}{|\gamma_{I_1}|\le\alpha}=\int_{-\alpha}^{\alpha}\frac{1}{\sqrt{2\pi}}e^{-\frac{x^2}{2}}dx<\sqrt{\frac{2}{\pi}}\alpha.
\end{equation}
Note that $C^n_k\le (\frac{en}{k})^k$. Thus,
\begin{equation}
\P{}{\inf_{I}|\gamma_I|\le \alpha}<(\frac{en}{k})^k\sqrt{\frac{2}{\pi}}\alpha=\sqrt{\frac{2}{\pi}}e^k(\frac{k}{n})^\epsilon\rightarrow 0.
\end{equation}
Thus $\exists N_2^{'}(k,\epsilon, \delta)$ such that $\forall n\ge N_2^{'}$, 
\begin{equation}\label{lower:gaussian}
\P{}{\inf_{I}|\gamma_I|\le \alpha}<\frac{\delta}{2}.
\end{equation}
To show the upper bound, we define $t_I=\ind{|\gamma_I|<\beta}$, where $\ind{}$ is the indicator function. Let $T=\sum_{I}t_I$. Then it suffices to show $\lim_{n\rightarrow \infty}\P{}{T=0}=0$.
Note that
\begin{equation}
\label{evbound}
\begin{aligned}
&\P{}{T=0}=\P{}{T-\E{}{T}=-\E{}{T}}\\
&\le\P{}{(T-\E{}{T})^2\ge(\E{}{T})^2}\le\frac{\V{T}}{(\E{}{T})^2},
\end{aligned}
\end{equation}
where the last inequality follows from the Markov inequality. 
Now we lower bound $\E{}{T}$.
\begin{equation}
\begin{aligned}
&\E{}{T}=\E{}{\sum_{I}t_I}=\sum_{I}\E{}{t_I}=C^{n}_k\E{}{t_{I_1}}\\
&=C^{n}_k\P{}{|\gamma_{I_1}|<\beta}=C^{n}_k\int_{-\beta}^{\beta}\frac{1}{\sqrt{2\pi}}e^{-\frac{x^2}{2}}dx.\\
\end{aligned}
\end{equation}
Note that $\epsilon<k$, thus $\beta<1$. For $x\in(-\beta,\beta)$, $\frac{1}{\sqrt{2\pi}}e^{-\frac{x^2}{2}}>\frac{1}{\sqrt{2\pi}}e^{-\frac{1}{2}}=\frac{1}{\sqrt{2\pi e}}$. Also note that $C^{n}_k>(\frac{n}{k})^k$, thus
\begin{equation}
\label{ebound}
\E{}{T}>(\frac{n}{k})^k\frac{1}{\sqrt{2\pi e}}2\beta=\sqrt{\frac{2}{\pi e}}(\frac{n}{k})^\epsilon.
\end{equation}
Now we upper bound $\V{T}$.
\begin{equation}
\V{T}=\sum_{I, I^{'}}\Cov{t_I}{t_{I^{'}}}=\sum_{I, I^{'}, |I\cap I^{'}|\ge1}\Cov{t_I}{t_{I^{'}}}.
\end{equation}
Note that for Bernoulli random variable $t_I$, $\V{t_I}\le \E{}{t_I}$. Thus if $I=I^{'}$, then
 \begin{equation}
 \begin{aligned}
 &\Cov{t_I}{t_{I^{'}}}=\V{t_I}\le \E{}{t_I}=\P{}{|\gamma_I|<\beta}\\
 &=\int_{-\beta}^{\beta}\frac{1}{\sqrt{2\pi}}e^{-\frac{x^2}{2}}dx<\frac{1}{\sqrt{2\pi}}2\beta=\sqrt{\frac{2}{\pi}}(\frac{k}{n})^{k-\epsilon}.
  \end{aligned}
 \end{equation}
Otherwise $1 \le |I\cap I^{'}| \le k-1$, that is, $I$ and $I^{'}$ overlap but not identical, then
\begin{equation}
\begin{aligned}
\Cov{t_I}{t_{I^{'}}}&=\E{}{t_It_{I^{'}}}-\E{}{t_I}\E{}{t_{I^{'}}}\le \E{}{t_It_{I^{'}}}\\
&=\P{}{|\gamma_I|<\beta, |\gamma_{I^{'}}|<\beta}.
\end{aligned}
\end{equation}
Note that $\gamma_I$ and $\gamma_{I^{'}}$ are jointly Gaussian with covariance
\begin{equation}
\begin{aligned}
\Cov{\gamma_I}{\gamma_{I^{'}}}&=\frac{1}{k}\sum_{i\in I, i^{'}\in I^{'}}\Cov{x_i-\theta^*}{x_{i^{'}}-\theta^*}\\
&=\frac{1}{k}\sum_{i\in I,i^{'}\in I^{'},i=i^{'}}1=\frac{|I\cap I^{'}|}{k}\defeq\rho,
\end{aligned}
\end{equation}
where $\frac{1}{k}\le\rho\le\frac{k-1}{k}$. The joint PDF of two standard normal distributions $x, y$ with covariance $\rho$ is
\begin{equation}
f(x,y)=\frac{1}{2\pi\sqrt{1-\rho^2}}e^{-\frac{x^2-2\rho xy+y^2}{2(1-\rho^2)}}.
\end{equation}
Note that $f(x, y)\le\frac{1}{2\pi\sqrt{1-\rho^2}}$, thus
\begin{equation}
\begin{aligned}
&\P{}{|\gamma_I|<\beta, |\gamma_{I^{'}}|<\beta}\le\iint\displaylimits_{|x|<\beta, |y|<\beta}\frac{1}{2\pi\sqrt{1-\rho^2}}dxdy\\
&=\frac{1}{2\pi\sqrt{1-\rho^2}}(2\beta)^2=\frac{2}{\pi\sqrt{1-\rho^2}}\beta^2.
\end{aligned}
\end{equation}
Since $\frac{2}{\pi\sqrt{1-\rho^2}}\le\frac{2}{\pi\sqrt{1-(\frac{k-1}{k})^2}}\le\frac{2}{\pi\sqrt{\frac{k}{k^2}}}=\frac{2\sqrt{k}}{\pi}$, thus
\begin{equation}
\P{}{|\gamma_I|<\beta, |\gamma_{I^{'}}|<\beta}\le\frac{2\sqrt{k}\beta^2}{\pi}=\frac{2\sqrt{k}}{\pi}(\frac{k}{n})^{2k-2\epsilon}.
\end{equation}
According to \lemref{lem:intersect}, there are at most $4^kC^{2k}_kC^n_{2k-1}$ pairs of $I$ and $I^{'}$ such that $1\le|I\cap I^{'}|\le k-1$. Thus,
\begin{equation}\label{vbound}
\begin{aligned}
&\V{T}=\sum_{I}\V{t_I}+\sum_{I\neq I^{'}, |I\cap I^{'}|\ge1}\Cov{t_I}{t_{I^{'}}}\\
&\le C^n_k\sqrt{\frac{2}{\pi}}(\frac{k}{n})^{k-\epsilon}+4^kC^{2k}_kC^n_{2k-1}\frac{2\sqrt{k}}{\pi}(\frac{k}{n})^{2k-2\epsilon}\\
&\le \sqrt{\frac{2}{\pi}}(\frac{en}{k})^k(\frac{k}{n})^{k-\epsilon}+4^kC^{2k}_k(\frac{en}{2k-1})^{2k-1}\frac{2\sqrt{k}}{\pi}(\frac{k}{n})^{2k-2\epsilon}\\
&=\sqrt{\frac{2}{\pi}}e^k(\frac{n}{k})^\epsilon+\frac{4\sqrt{k}}{\pi}C^{2k}_k(\frac{2ek}{2k-1})^{2k-1}(\frac{n}{k})^{2\epsilon-1}.\\
\end{aligned}
\end{equation}
Now plug~\eqref{vbound} and~\eqref{ebound} into~\eqref{evbound}, we have
\begin{equation}
\begin{aligned}
&\P{}{T=0}\le a_1(k)(\frac{n}{k})^{-\epsilon}+a_2(k)(\frac{n}{k})^{-1}\rightarrow 0,\\
\end{aligned}
\end{equation}
where $a_1=\sqrt{\frac{\pi}{2}}e^{k+1}$ and $a_2(k)=2\sqrt{k}eC^{2k}_k(\frac{2ek}{2k-1})^{2k-1}$. Thus $\exists N_2^{''}(k, \epsilon,\delta)$ such that $\forall n\ge N_2^{''}$,
\begin{equation}\label{upper:gaussian}
\P{}{\inf_{I}|\gamma_I|\ge \beta}<\frac{\delta}{2}.
\end{equation}
Let $N_2(k, \epsilon, \delta)=\max\{N_2^{'}(k, \epsilon,\delta), N_2^{''}(k, \epsilon,\delta)\}$, combining~\eqref{lower:gaussian} and ~\eqref{upper:gaussian} concludes the proof.
\end{proof}

Now we can conclude super-teaching by comparing~\thmref{thm:Gaussian} and~\propref{thm:AsyGaussianPool}:
\begin{proof}[\textbf{Proof of~\thmref{thm:Amuthm}}]
Let $\alpha=\frac{1}{\sqrt{k}}(\frac{k}{n})^{k-\epsilon}$ and $\beta=n^{-\frac{1}{2}-\epsilon}$. 
By \propref{thm:AsyGaussianPool}, $\forall \epsilon\in(0, \frac{2k-1}{4}),\forall\delta\in(0,1)$, $\exists N_1(\epsilon,\frac{\delta}{2})$ such that $\forall n\ge N_1$, $\P{}{R(\hat \theta_S)>\beta}> 1-\frac{\delta}{2}$. 
By \thmref{thm:Gaussian}, $\exists N_2(k, \epsilon,\frac{\delta}{2})$ such that $\forall n\ge N_2$, $\P{}{R(\hat \theta_{B_k(S)})<\alpha}>1-\frac{\delta}{2}$. Let $c_n=\frac{k^{k-\epsilon}}{\sqrt{k}}n^{-k+\frac{1}{2}+2\epsilon}$. 
Since $\epsilon<\frac{2k-1}{4}$, $c_n$ is a decreasing sequence in $n$ with $\lim_{n\rightarrow\infty}c_n=0$. Let $N_3(k, \epsilon)$ be the first integer such that $c_{N_3}\le1$.
Let $N(k, \epsilon, \delta)=\max\{N_1(\epsilon,\frac{\delta}{2}), N_2(k, \epsilon,\frac{\delta}{2}), N_3(k, \epsilon)\}$. By a union bound $\forall n\ge N(k, \epsilon, \delta)$, $\P{}{R(\hat\theta_{B_k(S)})<\alpha, R(\hat\theta_S)\ge \beta}>1-\delta$.
Since $\frac{\alpha}{\beta}=c_n$, we have $\P{}{R(\hat\theta_{B_k(S)})\le c_n R(\hat\theta_S)}>1-\delta$, where $c_n\le c_{N_3}\le1$.
\end{proof}

\section{Analysis on Super Teaching for 1D Large Margin Classifier}
\label{sec:midpoint}
We present our second theoretical result, this time on teaching a 1D large margin classifier. Let $\X=[-1,1]$, $\Y=\{-1,1\}$, $\Theta=[-1,1]$, $\theta^*=0$, 
$p_{\Z}(x, y)=p_{\Z}(x)p_{\Z}(y\mid x)$ where $p_{\Z}(x)=U(\X)$ and $p_{\Z}(y=1\mid x)=\ind{x\ge\theta^*}$. Let 
$x_- \defeq \max_{i:y_i=-1} x_i$ and
$x_+ \defeq \min_{i:y_i=+1} x_i$
be the inner-most pair of opposite labels in $S$ if they exist.
We formally define the large margin classifier  $A_{lm}(S)$ as
\begin{equation}
\hat\theta_S = A_{lm}(S)=\left\{
\begin{array}{ll}
(x_- + x_+)/2 & \mbox{ if $x_-$, $x_+$ exist} \\
-1 & \mbox{ if $S$ all positive} \\
1 & \mbox{ if $S$ all negative.}
\end{array}
\right.
\end{equation}
The risk is defined as $R(\hat\theta_S)=|\hat\theta_S-\theta^*|=|\hat\theta_S|$. The teacher we consider is the most symmetric teacher, who selects the most symmetric pair about $\theta^*$ in $S$ and gives it to the learner. 
We define the most-symmetric teacher $B_{ms}$:  
\begin{equation}
B_{ms}(S) = \left\{
\begin{array}{ll}
\{(s_-,-1), (s_+,1)\} & \mbox{ if } s_-,s_+ \mbox{ exist}, \\
\{(x_1,y_1)\} & \mbox{ otherwise.}
\end{array}
\right.
\label{eq:Bms}
\end{equation}
where $(s_-,s_+) \in \argmin_{(x,-1),(x',1) \in S} |\frac{x+x'}{2}-\theta^*|$.

Our main result shows that learning from the whole set $S$ achieves the well-known $O(1/n)$ risk, but surprisingly $B_{ms}$ achieves $O(1/n^2)$ risk, therefore it is an approximately $c_n=O(n^{-1})$ super teaching ratio.
\begin{restatable}{theorem}{Ampthm}
\label{thm:Ampthm}
Let $S$ be an $n$-item $iid$ sample drawn from $p_{\Z}$. Then $\forall \delta\in(0,1)$, $\exists N(\delta)$ such that $\forall n\ge N$, $\P{}{R(\hat\theta_{B_{ms}(S)})\le c_nR(\hat\theta_S)}>1-\delta$, where $c_n=\frac{32}{n\delta}\ln\frac{6}{\delta}$.
\end{restatable}

Before proving~\thmref{thm:Ampthm}, we first show that $B_{ms}$ is an optimal teacher for the large margin classifier.
\begin{restatable}{proposition}{propBms}
$B_{ms}$ is an optimal teacher for the large margin classifier $\hat\theta_S$.
\end{restatable}
\begin{proof}
We show $R(\hat\theta_{B_{ms}(S)})\le R(\hat\theta_{B(S)})$ for any $ B$ and any $ S$. 

If $|B_{ms}(S)|=1$, then $S$ is either all positive or all negative. 
In both cases $R(\hat\theta_{B(S)})=1$ for any $B$ by definition.
Thus $R(\hat\theta_{B_{ms}(S)})\le R(\hat\theta_{B(S)})$.

Otherwise $|B_{ms}(S)|=2$, then if $B(S)$ is all positive or all negative, we have $R(\hat\theta_{B(S)})=1$ and thus $R(\hat\theta_{B_{ms}(S)})\le R(\hat\theta_{B(S)})$. Otherwise let $x^B_-, x^B_+$ be the inner most pair of $B(S)$. Since $x^B_-, x^B_+\in S$, then by definition of $B_{ms}$, $R(\hat\theta_{B_{ms}(S)})=|\frac{s_-+s_+}{2}-\theta^*|\le |\frac{x^B_-+x^B_+}{2}-\theta^*|=R(\hat\theta_{B(S)})$.
\end{proof}
Now we show that learning on the whole $S$ incurs $O(n^{-1})$ risk. First, we give the following lemma for the exact tail probability of $R(\hat\theta_S)$.

\begin{restatable}{lemma}{lemAmp}
\label{lem:Amp}
For the large margin classifier $\hat\theta_S$, we have
\begin{equation}\label{eq:Amptail}
\P{}{R(\hat\theta_S)>\epsilon}=\left\{
\begin{aligned}
&(1-\epsilon)^n+(\epsilon)^n &&\text{ $0<\epsilon\le\frac{1}{2}$}\\
&(\frac{1}{2})^{n-1} &&\text{ $\frac{1}{2}<\epsilon<1$}\\
&0 &&\mbox{ $\epsilon=1$}.
\end{aligned}
\right.
\end{equation}
\end{restatable}
The proof for~\lemref{lem:Amp} is in the appendix.

Now we show that $R(\hat\theta_S)$ is $O(n^{-1})$.  
\begin{restatable}{theorem}{thmAmp}
\label{thm:Amp}
Let $S$ be an $n$-item $iid$ sample drawn from $p_{\Z}$. Then $\forall \delta\in(0,1)$ and $\forall n\ge2$,
\begin{equation}
\P{}{R(\hat\theta_S)>\frac{\delta}{n}}>1-\delta.
\end{equation}
\end{restatable}
\begin{proof}
According to~\lemref{lem:Amp}, for $\epsilon\le\frac{1}{2}$, we have
\begin{equation}\label{lowerlem2}
\P{}{R(\hat\theta_S)>\epsilon}>(1-\epsilon)^n>1-n\epsilon.
\end{equation}
Note that $n\ge2$, thus $\frac{\delta}{n}\le\frac{1}{2}$. Let $\epsilon=\frac{\delta}{n}$ in~\eqref{lowerlem2} we have
\begin{equation}
\begin{aligned}
&\P{}{R(\hat\theta_S)>\frac{\delta}{n}}>1-n\frac{\delta}{n}=1-\delta.
\end{aligned}
\end{equation}
\end{proof}
Now we work out the risk of the most symmetric teacher $B_{ms}$. 
To bound the risk of $B_{ms}$ we need the following key lemma, which shows that the sample complexity with the teacher is $O(\epsilon^{-1/2})$.

\begin{restatable}{lemma}{lemAmshighP}
\label{lem:Ams_highP}
Let $n=4m$, where $m$ is an integer. Let $S$ be an $n$-item $iid$ sample drawn from $p_{\Z}$. $\forall\epsilon>0, \forall\delta\in(0,1)$, $\exists \M(\epsilon, \delta)=\max\{\frac{3e}{\ln4-1}\ln\frac{3}{\delta}, (\frac{1}{\epsilon}\ln\frac{3}{\delta})^{\frac{1}{2}}\}$ such that $\forall m\ge\M (\epsilon, \delta)$, $\P{}{R(\hat\theta_{B_{ms}(S)})\le\epsilon}>1-\delta$.
\end{restatable}
\begin{proof}
We give a proof sketch and the details are in the appendix. Let $S_1=\{x\mid (x, 1)\in S\}$ and $S_2=\{x\mid (x, -1)\in S\}$ respectively. Then we have $|S_1|+|S_2|=4m$. Define event $E_1:\{|S_1|\ge m\land |S_2|\ge m\}$. Given that $m\ge \frac{3e}{\ln4-1}\ln\frac{3}{\delta}$, one can show $P(E_1)>1-\frac{\delta}{3}$. Since $|S_1|+|S_2|=4m$, either $|S_1|\ge2m$ or $|S_2|\ge2m$. Without loss of generality we assume $|S_1|\ge2m$. We then divide the interval [0, 1] equally into $N=\lfloor m^2(\ln\frac{3}{\delta})^{-1} \rfloor$ segments. The length of each segment is $\frac{1}{N}=O(\frac{1}{m^2})$ as Figure~\ref{segments_copy} shows.
\begin{figure}[H]
\centering
\includegraphics[width=3in,height=0.5in]{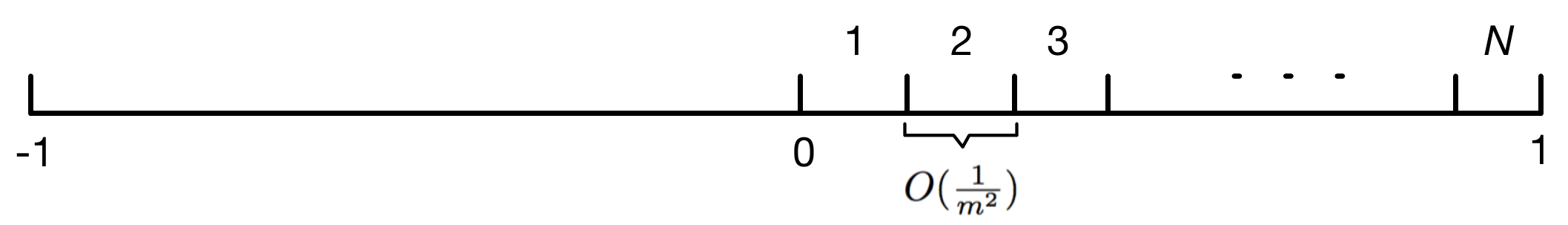}
\caption{segments}
\label{segments_copy}
\end{figure}
Let $N_o$ be the number of segments that are occupied by the points in $S_1$. Note that $N_o$ is a random variable.
Let $E_2$ be the event that $N_o\ge m$. Then one can show $P(E_2)>1-\frac{\delta}{3}$. By union bound, we have $P(E_1, E_2)>1-\frac{2\delta}{3}$. Let $E_3$ be the following event: there exist a point $x_2$ in $S_2$ such that $-x_2$, the flipped point, lies in the same segment as some point $x_1$ in $S_1$. One can show that $P(E_3\mid E_1, E_2)>1-\frac{\delta}{3}$. Thus $P(E_3)\ge P(E_1, E_2, E_3)=P(E_3\mid E_1, E_2)P(E_1, E_2)\ge (1-\frac{\delta}{3})(1-\frac{2\delta}{3})> 1-\delta$. If $E_3$ happens, then $|x_1+x_2|=|x_1-(-x_2)|\le\frac{1}{N}$. Note that $m\ge(\frac{1}{\epsilon}\ln\frac{3}{\delta})^{\frac{1}{2}}$ and $N=\lfloor m^2(\ln\frac{3}{\delta})^{-1} \rfloor\ge \frac{m^2}{2}(\ln\frac{3}{\delta})^{-1} $, thus $\frac{1}{N}\le \frac{2}{m^2}\ln\frac{3}{\delta}\le2\epsilon$. Therefore $R(\hat \theta_{B_{ms}(S)})=|\frac{s_-+s_+}{2}|\le|\frac{x_1+x_2}{2}|\le\epsilon$. 
\end{proof}

Rewriting $\epsilon$ in~\lemref{lem:Ams_highP} as a function of $n$, we have the following theorem.
\begin{restatable}{theorem}{thmAmsRes}
\label{thm:Ams}
Let $S$ be an $n$-item $iid$ sample dawn from $p_{\Z}$, then $\exists N_1(\delta)=\frac{12e}{\ln4-1}\ln\frac{3}{\delta}$ such that $\forall n\ge N_1$,
\begin{equation}
\P{}{R(\hat\theta_{B_{ms}(S)})\le\frac{16}{n^2}\ln\frac{3}{\delta}}>1-\delta.
\end{equation}
\end{restatable}
\begin{proof}
Note that if $n\ge N_1(\delta)=\frac{12e}{\ln4-1}\ln\frac{3}{\delta}$, then $m=\frac{n}{4}\ge\frac{3e}{\ln4-1}\ln\frac{3}{\delta}$, thus the minimum $\epsilon$ that satisfies $m\ge\M (\epsilon, \delta)$ is $\frac{1}{m^2}\ln\frac{3}{\delta}=\frac{16}{n^2}\ln\frac{3}{\delta}$.
\end{proof}
Now we can conclude super teaching:
\begin{proof}[\textbf{Proof of~\thmref{thm:Ampthm}}]
According to~\thmref{thm:Ams}, $\exists N_1(\frac{\delta}{2})$ such that $\forall n\ge N_1$, $\P{}{R(\hat\theta_{B_{ms}(S)})\le\frac{16}{n^2}\ln\frac{6}{\delta}}>1-\frac{\delta}{2}$. Note that $N_1\ge2$, thus according to~\thmref{thm:Amp}, $\forall n\ge N_1$, $\P{}{R(\hat\theta_S)>\frac{\delta}{2n}}>1-\frac{\delta}{2}$.   Let $c_n=\frac{32}{n\delta}\ln\frac{6}{\delta}$ and $N_2(\delta)=\frac{32}{\delta}\ln\frac{6}{\delta}$ so that $c_{N_2}=1$. Let $N(\delta)=\max\{N_1(\delta), N_2(\delta)\}$. By union bound, $\forall n\ge N$, with probability at least $1-\delta$, we have both $R(\hat\theta_S)>\frac{\delta}{2n}$ and $R(\hat\theta_{B_{ms}(S)})\le\frac{16}{n^2}\ln\frac{6}{\delta}$, which gives $\P{}{R(\hat\theta_{B_{ms}(S)})\le c_nR(\hat\theta_S)}>1-\delta$, where $c_n\le c_{N_2}=1$.
\end{proof}

\section{An MINLP Algorithm for Super Teaching}
\label{sec:MINLP}

Although the problem of proving super teaching ratios for a specific learner is interesting, we now focus on an algorithm to find a super teaching set for general learners \emph{given} a training set $S$.
That is, we find a subset $B(S) \subset S$ so that $R(\hat\theta_{B(S)}) < R(\hat\theta_S)$.
We start by formulating super teaching as a subset selection problem.
To this end, we introduce binary indicator variables $b_1, \ldots, b_n$ where $b_i=1$ means $z_i \in S$ is included in the subset.
We consider learners $A$ that can be defined via convex empirical risk minimization:
\begin{equation}\label{learner}
A(S) \defeq \argmin_{\theta \in \Theta} \sum_{i=1}^n \tilde\ell(\theta, z_i) + \frac{\lambda}{2} \|\theta\|^2.
\end{equation}
For simplicity we assume there is a unique global minimum which is returned by $\argmin$. 
Note that we use $\tilde\ell$ in~\eqref{learner} to denote the (surrogate) convex loss used by $A$ in performing empirical risk minimization.
For example, $\tilde\ell$ may be the negative log likelihood for logistic regression.
$\tilde\ell$ is potentially different from $\ell$ (e.g. the 0-1 loss) used by the teacher to define the teaching risk $R$ in~\eqref{eq:R}.

We formulate super teaching as the following bilevel combinatorial optimization problem:
\begin{eqnarray}
&&\min_{b\in\{0,1\}^n,\hat\theta\in\Theta} R(\hat\theta)\\
\mbox{s.t. } && \hat\theta = \argmin_{\theta \in \Theta} \sum_{i=1}^n b_i \tilde\ell(\theta, z_i) + \frac{\lambda}{2} \|\theta\|^2\label{eq:ML}. 
\end{eqnarray}
Under mild conditions, we may replace the lower level optimization problem (i.e. the machine learning problem~\eqref{eq:ML}) by its first order optimality (KKT) conditions:
\begin{eqnarray}
\min_{b\in\{0,1\}^n,\hat\theta\in\Theta} &&R(\hat\theta) \label{eq:MINLP}\\
\mbox{s.t. } && \sum_{i=1}^n b_i \nabla_\theta \tilde\ell(\hat\theta, z_i)  + {\lambda}\hat\theta = 0. \nonumber
\end{eqnarray}
This reduces the bilevel problem but the constraint is nonlinear in general, leading to 
a mixed-integer nonlinear program (MINLP), for which effective solvers exist.
We use the MINLP solver in NEOS~\cite{czyzyk1998neos}.

\section{Simulations}
\label{sec:exp}
We now apply the framework in section~\ref{sec:MINLP} to logistic regression and ridge regression, and show that the solver indeed selects a super-teaching subset that is far better than the original training set $S$. 

\subsection{Teaching Logistic Regression $A_{lr}$}
Let $\X=\R^d$, $\Theta=\R^d$, $\theta^* = (\frac{1}{\sqrt{d}},..., \frac{1}{\sqrt{d}})$, $p_{\Z}(x)=\N(0, I)$. Let $p_{\Z}(y\mid x)=\ind{x^\top\theta^*>0}$, which is deterministic given $x$.
Logistic regression estimates $\hat\theta_S=A_{lr}(S)$ with~\eqref{learner},
where 
$\lambda=0.1$ and
$\tilde\ell(z_i)=\log (1+\exp(-y_ix_i^\top\theta))$. 
In contrast,
The teacher's risk is defined to be the expected 0-1 loss: $R(\hat\theta)=\E{p_{\Z}}{\ind{\hat\theta(x)\neq y}}$, where $\hat\theta(x)$ is the label of $x$ predicted by $\hat\theta$. 
Since $p_{\Z}$ is symmetric about the origin, the risk can be rewritten in terms of the angle between $\hat\theta$ and $\theta^*$: $R(\hat\theta)=\arccos(\frac{\hat\theta^\top\theta^*}{||\hat\theta||\cdot||\theta^*||})/\pi$. 
Instantiating~\eqref{eq:MINLP} we have
\begin{eqnarray}
\min_{b\in\{0,1\}^n,\hat\theta\in\R^d} &&\arccos(\frac{\hat\theta^\top\theta^*}{||\hat\theta||\cdot||\theta^*||})/\pi \label{eq:Logistic}\\
\mbox{s.t. } &&  \lambda\hat\theta-\sum_{i=1}^n \frac{b_iy_ix_i}{1+\exp(y_ix_i^\top\hat\theta)} = 0. \nonumber
\end{eqnarray}

We run experiments to study the effectiveness and scalability of the NEOS MINLP solver on~\eqref{eq:Logistic}, specifically with respect to the training set size $n=|S|$ and dimension $d$. 

In the first set of experiments we fix $d=2$ and vary $n=16, 64, 256$ and $1024$. 
For each $n$ we run 10 trials.
In each trial we draw an $n$-item $iid$ sample $S \sim p_{\Z}$ and call the solver on~\eqref{eq:Logistic}.
The solver's solution to $b_1 \ldots b_n$ indicates the super teaching set $B(S)$.
We then compute an empirical version of the super teaching ratio:
$$\hat c_n=R(\hat\theta_{B(S)})/R(\hat\theta_S).$$ 

\tabcolsep=0.09cm
\begin{table}[ht]
	\small
	\centering
	\begin{tabular}{ |c|c|c|c|c|c|c|} 
		\hline
		&  \multicolumn{3}{|c| }{Logistic Regression} &  \multicolumn{3}{ |c| }{Ridge Regression} \\
		\hline
		$n=|S|$ &$\hat c_n $ &  $|B(S)|/n$ & time (s) &  $\hat c_n $ & $|B(S)|/n$ &time (s)\\ 
		\hline
		16 & 8.5e-4&0.50&3.4e-1
		&7.8e-3&0.50&6.3e-1\\ 
		64 &1.3e-3&0.69&3.5e+0
		&7.5e-3&0.70&5.8e+0\\
		256 &6.3e-3&0.67&6.0e+1
		&5.6e-3&0.84&1.4e+2  \\
		1024 &1.3e-2&0.86&1.4e+3
		&4.1e-3&0.92& 3.3e+3\\
		\hline
	\end{tabular}
		\caption{Super teaching as $n$ changes.}\label{Ep:empc}
\end{table}

\begin{table}[ht]
	\small
	\centering
	\begin{tabular}{ |c|c|c|c|c|c|c|} 
		\hline
		&  \multicolumn{3}{|c| }{Logistic Regression} &  \multicolumn{3}{ |c| }{Ridge Regression} \\
		\hline
		$d$ &$\hat c_n$ &  $|B(S)|/n$ & time (s) &  $\hat c_n $ & $|B(S)|/n$ &time (s)\\ 
		\hline
		2 & 3.1e-3 &0.67 &5.4e-1 &
		3.3e-3 &0.55 &6.6e+0\\ 
		4 & 2.4e-3 &0.44&8.5e+1 &
		7.2e-3 &0.53 &5.8e+1   \\
		8 & 1.8e-1&0.39 &4.1e+0 &
		1.5e-1 &0.47 &6.0e+0  \\
		16 & 5.6e-1&0.42 &5.1e+0 &
		4.3e-1 &0.59 &9.3e+0  \\
		32 & 8.2e-1&0.58 &1.0e+1 &
		6.4e-1 &0.86 &3.0e+0  \\
		\hline
	\end{tabular}
	\caption{Super teaching as $d$ changes.}\label{Ep:empd}
\end{table}

In the left half of Table~\ref{Ep:empc} we report the median of the following quantities over 10 trials: $\hat c_n$, the fraction of the training items selected for super teaching $|B(S)|/n$, and the NEOS server running time.

The main result is that $\hat c_n\ll1$ for all $n$, which means the solver indeed selects a super-teaching set $B(S)$ that is far better than the original $iid$ training set $S$. 
Therefore, MINLP is a valid algorithm for finding a super teaching set.

Second, we note that the solver tends to select a large subset since the median $|B(S)|/n \ge 1/2$.
This is interesting as it is known that when $S$ is dense, one can select extremely sparse super teaching sets, as small as a few items, to teach effectively~\cite{JMLR:v17:15-630}.
Understanding the different regimes remains future work. 

Finally, the running time grows fast with $n$.
For example, when $n=1024$ it takes around half an hour to solve~\eqref{eq:Logistic}.
Future work needs to address this bottleneck in applying MINLP to large problems.

In the second set of experiments we fix $n=32$ and vary $d = 2, 4, 8, 16, 32$.
The left half of Table~\ref{Ep:empd} shows the results.
The empirical teaching ratio $\hat c_n$ is still below 1 in all cases, showing super teaching.
But as the dimension of the problem increases $\hat c_n$ deteriorates toward 1.
Nonetheless, even when $d=n$ we still see a median super teaching ratio of 0.82; the corresponding super teaching set $B(S)$ has only 58\% training items than the dimension.
It is interesting that the MINLP algorithm intentionally created a ``high dimensional'' learning problem (as in higher dimension $d$ than selected training items $|B(S)|$) to achieve better teaching, knowing that the learner $A_{lr}$ is regularized.
The running time does not change dramatically.

\subsection{Teaching Ridge Regression $A_{rr}$}
Let $\X=\R^d$, $\Theta=\R^d$, $\theta^* = (\frac{1}{\sqrt{d}},..., \frac{1}{\sqrt{d}})$, $p_{\Z}(x)=\N(0, I)$, 
$p_{\Z}(y\mid x)=\N(y; x^\top\theta^*, 0.1)$. 
Let the teaching risk be the parameter difference: $R(\hat\theta)=\|\hat\theta-\theta^*\|$.
Given a sample $S$ with $n$ $iid$ items drawn from $p_{\Z}$,
ridge regression estimates $\hat\theta_S=A_{rr}(S)$ with $\lambda=0.1$ and
$\tilde\ell(z_i)=(x_i^\top\hat\theta-y_i)^2$. 
The corresponding MINLP is:
\begin{eqnarray}
\min_{b\in\{0,1\}^n,\hat\theta\in\R^d} &&||\hat\theta-\theta^*|| \label{eq:Ridge}\\
\mbox{s.t. } &&  \lambda\hat\theta+2\sum_{i=1}^n b_i(x_i^\top\hat\theta- y_i) x_i = 0. \nonumber
\end{eqnarray}
We run the same set of experiments. 
Tables ~\ref{Ep:empc} and~\ref{Ep:empd} show the results,
which are qualitatively similar to teaching logistic regression.
Again, we see the empirical super teaching ratio $\hat c_n \ll 1$, indicating the presence of super teaching.

\begin{figure}[ht]
	\centering
	\begin{subfigure}[t]{0.45\columnwidth}
		\centering
		\includegraphics[width=.9\columnwidth]{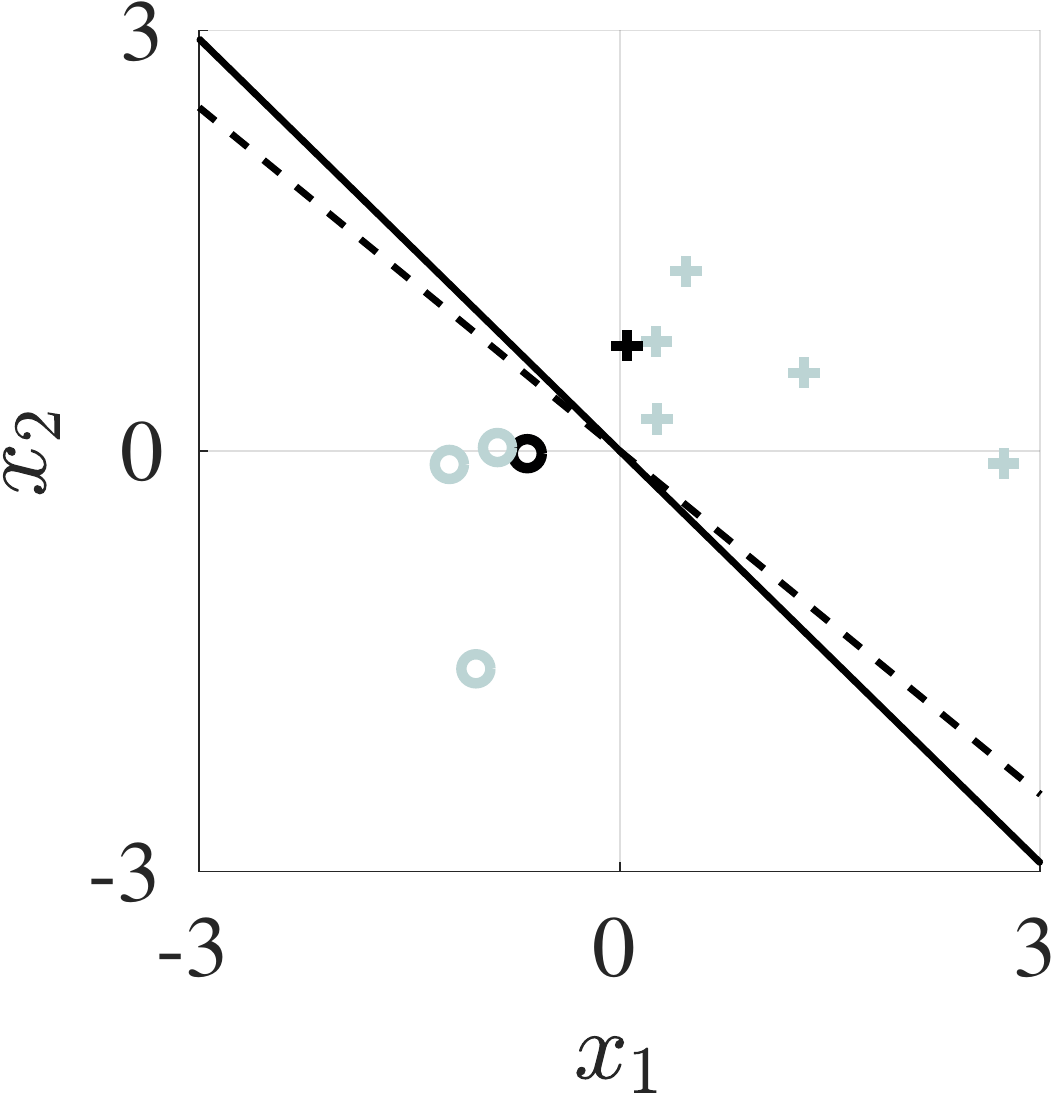}
		\caption{logistic regression}
		\label{fig:logistic}
	\end{subfigure}
	~
	\begin{subfigure}[t]{0.45\columnwidth}
		\centering
		\includegraphics[width=.9\columnwidth]{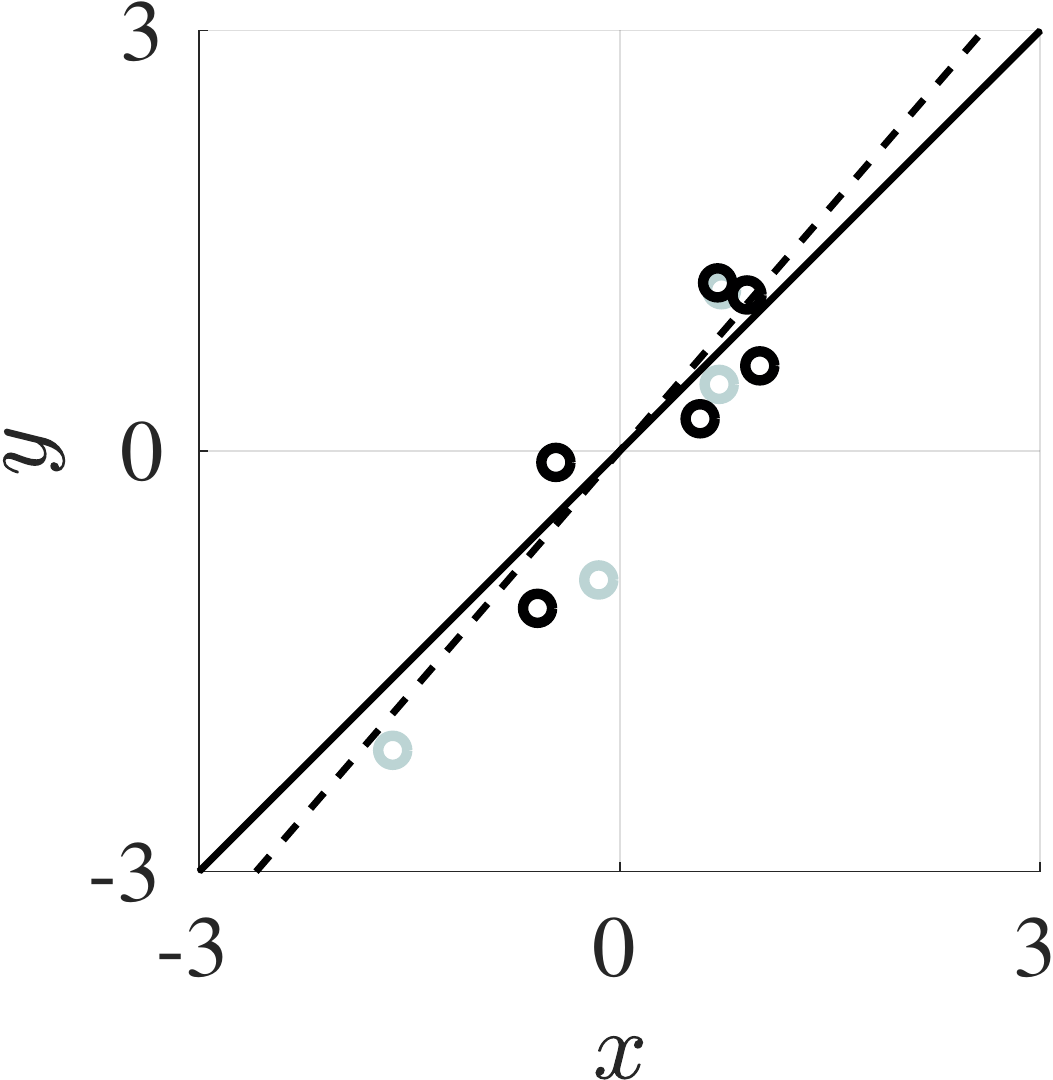} 
		\caption{ridge regression}
		\label{fig:ridge}
	\end{subfigure}%
	\caption{Typical trials from the MINLP algorithm}
	\label{fig:example}
\end{figure}

Finally, Figure~\ref{fig:example} visualizes one typical trial each for teaching logistic regression and ridge regression.	
$S$ consists of both dark and light points, while the dark ones representing $B(S)$ optimized by MINLP. 
The dashed line shows $\hat\theta_S$, while the solid lines shows $\hat\theta_{B(S)}$. 
The ground truth 
($x_1+x_2=0$ in logistic regression, $y=x$ in ridge regression) 
essentially overlaps with the solid lines.
Specifically, the super taught models $\hat\theta_{B(S)}$ have negligible risks of 2.5e-4 and 3.3e-3, whereas models $\hat\theta_S$ trained from the whole $iid$ sample $S$ incur much larger risks of 0.03 and 0.16, respectively.

\section{Related Work}
\label{sec:relatedwork}
There has been several research threads in different communities aimed at reducing a data set while maintaining its utility.
The first thread is training set reduction~\cite{garcia2012prototype,zeng2005smo,Wilson2000}, which during training time prunes items in $S$ in an attempt to improve the learned model.
The second thread is coresets~\cite{har2011geometric,2017arXiv170306476B}, a summary of $S$ such that models learned on the summary are provably competitive with models learned on the full data set $S$.
But as they do not know the target model $p_{\Z}$ or $\theta^*$, these methods cannot truly achieve super teaching.
The third thread is curriculum learning~\cite{icml2009_006} which showed that smart initialization is useful for nonconvex optimization.
In contrast, our teacher can directly encode the true model and therefore obtain faster rates.
The final thread is sample compression~\cite{floyd1995sample}, where a compression function chooses a subset $T\subset S$ and a reconstruction function to form a hypothesis.
Our present work has some similarity with compression, which allows increased accuracy since compression bounds can be used as regularization~\cite{kontorovich2017nearest}.

The theoretical study of machine teaching has focused on the teaching dimension, i.e. the minimum training set size needed to exactly teach a target concept $\theta^*$~\cite{Goldman1995Complexity,Shinohara1991Teachability,Zhu2017NoLearner,JMLR:v18:16-460,Liu2016Teaching,Zhu2015Machine,JMLR:v15:doliwa14a,zhu2013machine,Zilles2011Models,Balbach2009Recent,982362,conf/colt/AngluinK97,Goldman1996Teaching,DBLP:journals/jcss/Mathias97,Balbach2006Teaching,Balbach:2008:MTU:1365093.1365255,Kobayashi2009Complexity,journals/ml/AngluinK03,conf/colt/RivestY95,Hegedus1995Generalized,journals/ml/Ben-DavidE98}. Most of the prior work assumed a synthetic teaching setting where $S$ is the whole item space, which is often unrealistic. Liu \textit{et al.} considered approximate teaching in the finite $S$ setting~\cite{Liu2017Iterative}, though their analysis focused on a specific SGD learner. 
Our super teaching setting applies to arbitrary learners, and we allow approximate teaching -- namely we do not require the teacher to teach exactly the target model, which is infeasible in our pool-based teaching setting with a finite $S$.

Machine teaching applications include education~\cite{Clement2016edm,Patil2014Optimal,singla2014near,NIPS2013_4887,Cakmak2011Mixed,Rafferty:2011:FTP:2026506.2026545}, computer security~\cite{Alfeld2017Explicit,Alfeld2016Data,Mei2015Machine}, and interactive machine learning~\cite{Suh2016Label,AAAI124954,Khan2011How}.
By establishing the existence of super-teaching, the present paper can guide the process of finding a more effective training set for these applications.

\section{Discussions and Conclusion}
\label{discuss:superteach}
We presented super-teaching: when the teacher already knows the target model, she can often choose from a given training set a smaller subset that trains a learner better.
We proved this for two learners, and provided an empirical algorithm based on mixed integer nonlinear programming to find a super teaching set.

However, much needs to be done on the theory of super teaching.  
We give two counterexamples to illustrate that not all learners are super-teachable.
\begin{example}[MLE of interval]
Let $\X=[0, \theta^*]$, where $\theta^*\in \R^+$. $p_{\Z}(x)=U(\X)$. Given a $n$-item training set $S$, the MLE for $\theta^*$ is $\hat\theta_S=A_{int}(S)=\max_{i=1:n}x_i$. The risk is defined as $R(\hat\theta_S)=|\hat\theta_S-\theta^*|$.
We show $A_{int}$ is not super-teachable.
$\hat\theta_{B(S)}=\max_{x_i\in B(S)}x_i\le \max_{x_i\in S}x_i=\hat\theta_S$. 
Since $\hat\theta_S \le \theta^*$, $R(\hat\theta_{B(S)})=|\hat\theta_{B(S)}-\theta^*|\ge |\hat\theta_S-\theta^*|=R(\hat\theta_S)$.
\end{example}

We can generalize this to a classification setting, and show that neither the least nor the greatest consistent hypothesis is not super-teachable:
\begin{example}[Consistent learners]\label{ex:consistent}
Let $\X=[x_{\min}, x_{\max}] \subset \Z$ be an interval over the integer grid.
The hypothesis space is $\Theta=\{[a,b]\subseteq\X: \mbox{$y=1$ in $[a,b]$ and $-1$ outside}\}$.
$\theta^*=[a^*, b^*] \in \Theta$.
$p_{\Z}$ is uniform on $\X$ and noiseless $y$ labeled according to $\theta^*$.
The risk $R(\hat\theta_S)$ is the size of the symmetric difference between the two intervals $\hat\theta_S$ and $\theta^*$, normalized by $x_{\max} - x_{\min}$.
Given a sample $S$, the least consistent learner $A_{lc}$ learns the tightest interval over positive items in $S$:
$\hat\theta^{lc}_S=A_{lc}(S) \defeq \left[\min_{\substack{i=1:n\\y_i=1}} x_i, \max_{\substack{i=1:n\\y_i=1}} x_i \right].$
$\hat\theta^{lc}_S=\emptyset$ if $S$ does not contain positive items.
The greatest consistent learner $A_{gc}$ extends the hypothesis interval 
in both directions as much as possible before hitting negative points in $S$.
If $S$  has no positive we define $\hat\theta^{gc}_S=\emptyset$, too.

\begin{restatable}{proposition}{Consistent}
Neither $A_{lc}$ nor $A_{gc}$ is super-teachable.
\end{restatable}
\begin{proof}
We first show $A_{lc}$ is not super-teachable. Note that $A_{lc}$ learns the tightest interval consistent with $S$, thus we always have $\hat\theta^{lc}_S\subseteq \theta^*$. Now we show that $\hat\theta^{lc}_{B(S)}\subseteq \hat\theta^{lc}_S$ is always true so that $R(\hat\theta^{lc}_S)\le R(\hat\theta^{lc}_{B(S)})$ follows.

If $\theta^*=\emptyset$, then trivially $\hat\theta^{lc}_{B(S)}= \hat\theta^{lc}_S=\emptyset$.

Now assume $\theta^*\neq\emptyset$.
If $\exists (x, 1)\in B(S)$, let $[a_1, b_1]=\hat\theta^{lc}_{B(S)}$. Note that $\hat\theta^{lc}_S\neq \emptyset$ because $B(S)\subseteq S$ and thus $S$ has at least one positive point. Let $\hat\theta^{lc}_S=[a_2,b_2]$. 
Now $a_1=\min\{x\mid (x, 1)\in B(S)\}\ge \min\{x\mid (x, 1)\in S\}=a_2$, and $b_1=\max\{x\mid (x, 1)\in B(S)\}\le \max\{x\mid (x, 1)\in S\}=b_2$. Thus we have $\hat\theta^{lc}_{B(S)}\subseteq \hat\theta^{lc}_S$.
If $\nexists (x, 1)\in B(S)$, $\hat\theta^{lc}_{B(S)}=\emptyset$ and $\hat\theta^{lc}_{B(S)}\subseteq \hat\theta^{lc}_S$ is always true.

Thus $\hat\theta^{lc}_{B(S)}\subseteq \hat\theta^{lc}_S\subseteq \theta^*$ for any $B$ and any $S$.

The proof for $A_{gc}$ is similar by showing $\theta^*\subseteq \hat\theta^{gc}_S\subseteq \hat\theta^{gc}_{B(S)}$.
\end{proof}
\end{example}


This leads to an open question: which family of learners are super teachable?
We offer a conjecture here:
we speculate that MLEs (and the derived MAP estimates or regularized empirical risk minimizers) which satisfy the asymptotic normality conditions~\cite{white1982maximum} are super teachable.
This conjecture is motivated by its similarity to the proof in section~\ref{sec:Gaussian}.
Also note that the two counterexamples are classic examples of MLE that do \emph{not} satisfy the asymptotic normality conditions. 

Another open question concerns the optimal super-teaching subset size $k$ for a given training set of size $n$. For example, our result on teaching the MLE of Gaussian mean indicates that the rate improves as $k$ grows.
However, our analysis only applies to a fixed $k$.
Further research is needed to identify the optimal $k$.

\textbf{Acknowledgments}: R.N. acknowledges support by NSF IIS-1447449 and CCF-1740707. P.R. is supported in part by grants NSF DMS-1712596, NSF DMS-TRIPODS-1740751, DARPA W911NF-16-1-0551, ONR N00014-17-1-2147 and a grant from the MIT NEC Corporation.
X.Z. is supported in part by NSF CCF-1704117, IIS-1623605, CMMI-1561512, DGE-1545481, and CCF-1423237. 

\bibliography{zhu}

\begin{thebibliography}{10}

\bibitem{Alfeld2016Data}
S.~Alfeld, X.~Zhu, and P.~Barford.
\newblock Data poisoning attacks against autoregressive models.
\newblock {\em AAAI}, 2016.

\bibitem{Alfeld2017Explicit}
S.~Alfeld, X.~Zhu, and P.~Barford.
\newblock Explicit defense actions against test-set attacks.
\newblock In {\em The Thirty-First AAAI Conference on Artificial Intelligence
  (AAAI)}, 2017.

\bibitem{982362}
D.~Angluin.
\newblock Queries revisited.
\newblock {\em Theoretical Computer Science}, 313(2):175--194, 2004.

\bibitem{conf/colt/AngluinK97}
D.~Angluin and M.~Krikis.
\newblock Teachers, learners and black boxes.
\newblock {\em COLT}, 1997.

\bibitem{journals/ml/AngluinK03}
D.~Angluin and M.~Krikis.
\newblock Learning from different teachers.
\newblock {\em Machine Learning}, 51(2):137--163, 2003.

\bibitem{2017arXiv170306476B}
O.~{Bachem}, M.~{Lucic}, and A.~{Krause}.
\newblock {Practical Coreset Constructions for Machine Learning}.
\newblock {\em ArXiv e-prints}, Mar. 2017.

\bibitem{Balbach:2008:MTU:1365093.1365255}
F.~J. Balbach.
\newblock Measuring teachability using variants of the teaching dimension.
\newblock {\em Theor. Comput. Sci.}, 397(1-3):94--113, 2008.

\bibitem{Balbach2006Teaching}
F.~J. Balbach and T.~Zeugmann.
\newblock Teaching randomized learners.
\newblock {\em COLT}, pages 229--243, 2006.

\bibitem{Balbach2009Recent}
F.~J. Balbach and T.~Zeugmann.
\newblock Recent developments in algorithmic teaching.
\newblock In {\em Proceedings of the 3rd International Conference on Language
  and Automata Theory and Applications}, pages 1--18, 2009.

\bibitem{journals/ml/Ben-DavidE98}
S.~Ben-David and N.~Eiron.
\newblock Self-directed learning and its relation to the {VC}-dimension and to
  teacher-directed learning.
\newblock {\em Machine Learning}, 33(1):87--104, 1998.

\bibitem{icml2009_006}
Y.~Bengio, J.~Louradour, R.~Collobert, and J.~Weston.
\newblock Curriculum learning.
\newblock In L.~Bottou and M.~Littman, editors, {\em Proceedings of the 26th
  International Conference on Machine Learning}, pages 41--48, Montreal, June
  2009. Omnipress.

\bibitem{AAAI124954}
M.~Cakmak and M.~Lopes.
\newblock Algorithmic and human teaching of sequential decision tasks.
\newblock In {\em AAAI}, 2012.

\bibitem{Cakmak2011Mixed}
M.~Cakmak and A.~Thomaz.
\newblock Mixed-initiative active learning.
\newblock {\em ICML Workshop on Combining Learning Strategies to Reduce Label
  Cost}, 2011.

\bibitem{Clement2016edm}
B.~Clement, P.-Y. Oudeyer, and M.~Lopes.
\newblock A comparison of automatic teaching strategies for heterogeneous
  student populations.
\newblock In {\em Educational Data Mining (EDM)}, 2016.

\bibitem{czyzyk1998neos}
J.~Czyzyk, M.~P. Mesnier, and J.~J. Mor{\'e}.
\newblock The {NEOS} server.
\newblock {\em IEEE Computational Science and Engineering}, 5(3):68--75, 1998.

\bibitem{JMLR:v15:doliwa14a}
T.~Doliwa, G.~Fan, H.~U. Simon, and S.~Zilles.
\newblock Recursive teaching dimension, {VC}-dimension and sample compression.
\newblock {\em Journal of Machine Learning Research}, 15:3107--3131, 2014.

\bibitem{floyd1995sample}
S.~Floyd and M.~Warmuth.
\newblock Sample compression, learnability, and the {V}apnik-{C}hervonenkis
  dimension.
\newblock {\em Machine learning}, 21(3):269--304, 1995.

\bibitem{JMLR:v18:16-460}
Z.~Gao, C.~Ries, H.~U. Simon, and S.~Zilles.
\newblock Preference-based teaching.
\newblock {\em Journal of Machine Learning Research}, 18(31):1--32, 2017.

\bibitem{garcia2012prototype}
S.~Garcia, J.~Derrac, J.~Cano, and F.~Herrera.
\newblock Prototype selection for nearest neighbor classification: Taxonomy and
  empirical study.
\newblock {\em IEEE Transactions on Pattern Analysis and Machine Intelligence},
  34(3):417--435, 2012.

\bibitem{Goldman1995Complexity}
S.~Goldman and M.~Kearns.
\newblock On the complexity of teaching.
\newblock {\em Journal of Computer and Systems Sciences}, 50(1):20--31, 1995.

\bibitem{Goldman1996Teaching}
S.~A. Goldman and H.~D. Mathias.
\newblock Teaching a smarter learner.
\newblock {\em Journal of Computer and Systems Sciences}, 52(2):255--267, 1996.

\bibitem{har2011geometric}
S.~Har-Peled.
\newblock {\em Geometric approximation algorithms}, volume 173.
\newblock American mathematical society Boston, 2011.

\bibitem{Hegedus1995Generalized}
T.~Heged\"us.
\newblock Generalized teaching dimensions and the query complexity of learning.
\newblock In {\em Proceedings of the eighth Annual Conference on Computational
  Learning Theory (COLT)}, pages 108--117, 1995.

\bibitem{Khan2011How}
F.~Khan, X.~Zhu, and B.~Mutlu.
\newblock How do humans teach: On curriculum learning and teaching dimension.
\newblock {\em NIPS}, 2011.

\bibitem{Kobayashi2009Complexity}
H.~Kobayashi and A.~Shinohara.
\newblock Complexity of teaching by a restricted number of examples.
\newblock {\em COLT}, pages 293--302, 2009.

\bibitem{kontorovich2017nearest}
A.~Kontorovich, S.~Sabato, and R.~Weiss.
\newblock Nearest-neighbor sample compression: Efficiency, consistency,
  infinite dimensions.
\newblock {\em arXiv preprint arXiv:1705.08184}, 2017.

\bibitem{NIPS2013_4887}
R.~Lindsey, M.~Mozer, W.~J. Huggins, and H.~Pashler.
\newblock Optimizing instructional policies.
\newblock In C.~Burges, L.~Bottou, M.~Welling, Z.~Ghahramani, and
  K.~Weinberger, editors, {\em Advances in Neural Information Processing
  Systems 26}, pages 2778--2786. 2013.

\bibitem{JMLR:v17:15-630}
J.~Liu and X.~Zhu.
\newblock The teaching dimension of linear learners.
\newblock {\em Journal of Machine Learning Research}, 17(162):1--25, 2016.

\bibitem{Liu2016Teaching}
J.~Liu, X.~Zhu, and H.~G. Ohannessian.
\newblock The teaching dimension of linear learners.
\newblock In {\em The 33rd International Conference on Machine Learning
  (ICML)}, 2016.

\bibitem{Liu2017Iterative}
W.~Liu, B.~Dai, J.~M. Rehg, and L.~Song.
\newblock Iterative machine teaching.
\newblock In {\em ICML}, 2017.

\bibitem{DBLP:journals/jcss/Mathias97}
H.~D. Mathias.
\newblock A model of interactive teaching.
\newblock {\em J. Comput. Syst. Sci.}, 54(3):487--501, 1997.

\bibitem{Mei2015Machine}
S.~Mei and X.~Zhu.
\newblock Using machine teaching to identify optimal training-set attacks on
  machine learners.
\newblock {\em AAAI}, 2015.

\bibitem{Patil2014Optimal}
K.~Patil, X.~Zhu, L.~Kopec, and B.~C. Love.
\newblock Optimal teaching for limited-capacity human learners.
\newblock {\em Advances in Neural Information Processing Systems (NIPS)}, 2014.

\bibitem{Rafferty:2011:FTP:2026506.2026545}
A.~N. Rafferty, E.~Brunskill, T.~L. Griffiths, and P.~Shafto.
\newblock Faster teaching by {POMDP} planning.
\newblock In {\em Proceedings of the 15th International Conference on
  Artificial Intelligence in Education}, AIED'11, pages 280--287, Berlin,
  Heidelberg, 2011. Springer-Verlag.

\bibitem{ribeiro2016should}
M.~T. Ribeiro, S.~Singh, and C.~Guestrin.
\newblock Why should i trust you?: Explaining the predictions of any
  classifier.
\newblock In {\em Proceedings of the 22nd ACM SIGKDD International Conference
  on Knowledge Discovery and Data Mining}, pages 1135--1144. ACM, 2016.

\bibitem{conf/colt/RivestY95}
R.~L. Rivest and Y.~L. Yin.
\newblock Being taught can be faster than asking questions.
\newblock {\em COLT}, 1995.

\bibitem{Shalev-Shwartz:2014:UML:2621980}
S.~Shalev-Shwartz and S.~Ben-David.
\newblock {\em Understanding Machine Learning: From Theory to Algorithms}.
\newblock Cambridge University Press, New York, NY, USA, 2014.

\bibitem{Shinohara1991Teachability}
A.~Shinohara and S.~Miyano.
\newblock Teachability in computational learning.
\newblock {\em New Generation Computing}, 8(4):337–--348, 1991.

\bibitem{singla2014near}
A.~Singla, I.~Bogunovic, G.~Bartok, A.~Karbasi, and A.~Krause.
\newblock Near-optimally teaching the crowd to classify.
\newblock In {\em ICML}, pages 154--162, 2014.

\bibitem{Suh2016Label}
J.~Suh, X.~Zhu, and S.~Amershi.
\newblock The label complexity of mixed-initiative classifier training.
\newblock {\em International Conference on Machine Learning (ICML)}, 2016.

\bibitem{white1982maximum}
H.~White.
\newblock Maximum likelihood estimation of misspecified models.
\newblock {\em Econometrica: Journal of the Econometric Society}, pages 1--25,
  1982.

\bibitem{Wilson2000}
D.~R. Wilson and T.~R. Martinez.
\newblock Reduction techniques for instance-based learning algorithms.
\newblock {\em Machine Learning}, 38(3):257--286, 2000.

\bibitem{zeng2005smo}
X.~Zeng and X.-w. Chen.
\newblock {SMO}-based pruning methods for sparse least squares support vector
  machines.
\newblock {\em IEEE transactions on Neural Networks}, 16(6):1541--1546, 2005.

\bibitem{zhu2013machine}
X.~Zhu.
\newblock Machine teaching for bayesian learners in the exponential family.
\newblock In {\em Advances in Neural Information Processing Systems (NIPS)},
  pages 1905--1913, 2013.

\bibitem{Zhu2015Machine}
X.~Zhu.
\newblock Machine teaching: an inverse problem to machine learning and an
  approach toward optimal education.
\newblock {\em AAAI}, 2015.

\bibitem{Zhu2017NoLearner}
X.~Zhu, J.~Liu, and M.~Lopes.
\newblock No learner left behind: On the complexity of teaching multiple
  learners simultaneously.
\newblock In {\em The 26th International Joint Conference on Artificial
  Intelligence (IJCAI)}, 2017.

\bibitem{Zhu2018Overview}
X.~Zhu, A.~Singla, S.~Zilles, and A.~N. Rafferty.
\newblock {An Overview of Machine Teaching}.
\newblock {\em ArXiv e-prints}, Jan. 2018.
\newblock https://arxiv.org/abs/1801.05927.

\bibitem{Zilles2011Models}
S.~Zilles, S.~Lange, R.~Holte, and M.~Zinkevich.
\newblock Models of cooperative teaching and learning.
\newblock {\em Journal of Machine Learning Research}, 12:349--384, 2011.

\end{thebibliography}
\bibliographystyle{abbrv}

\onecolumn
\newpage
\section*{Supplemental Material}

\lemAmp*
\begin{proof}
The risk is $R(\hat\theta_S)=|\hat\theta_S|$. Define event $E: \{\exists (x,-1)\in S\land \exists (x,+1)\in S\}$. $\P{}{|\hat\theta_S|\le\epsilon}$ can be decomposed into two components depending on if $E$ happens as \eqref{eq:Amp_divide} shows. 
\begin{equation}\label{eq:Amp_divide}
\P{}{|\hat\theta_S|\le\epsilon}=\P{}{|\hat\theta_S|\le\epsilon, E}+\P{}{|\hat\theta_S|\le\epsilon,E^c}.
\end{equation}
$\P{}{|\hat\theta_S|\le\epsilon,E^c}=\P{}{|\hat\theta_S|\le\epsilon\mid E^c}\P{}{E^c}$. Note that $\P{}{E^c}=(\frac{1}{2})^{n-1}$, $\P{}{|\hat\theta_S|\le\epsilon\mid E^c}$ is 0 if $\epsilon<1$ and 1 if $\epsilon=1$ because $\hat\theta_S=\pm1$ always holds given $E^c$ happens. Thus,
\begin{equation}\label{eq:Ec}
\P{}{|\hat\theta_S|\le\epsilon, E^c}=\left\{
\begin{aligned}
&0 && \mbox{if $\epsilon<1$}\\
&(\frac{1}{2})^{n-1} && \mbox{if $\epsilon=1$}.
\end{aligned}
\right.
\end{equation}
Now we compute $\P{}{|\hat\theta_S|\le\epsilon, E}$. Let $n_+$ be the number of positive points in $S$. Define $E_i: \{n_+=i\}$. Note $E_i\cap E_j=\emptyset$ if $i\neq j$ and $E=\cup_{i=1}^{n-1}E_i$, thus 
\begin{equation}\label{divide}
\P{}{|\hat\theta_S|\le\epsilon, E}=\sum_{i=1}^{n-1}\P{}{|\hat\theta_S|\le\epsilon, E_i}=\sum_{i=1}^{n-1}\P{}{|\hat\theta_S|\le\epsilon \mid E_i}\P{}{E_i}.
\end{equation}
$\P{}{E_i}=C^n_i(\frac{1}{2})^n$. Note that $\P{}{|\hat\theta_S|\le\epsilon \mid E_i}=\P{}{|\frac{x_-+x_+}{2}|\le\epsilon \mid E_i}$. To compute it, we first compute $F_{-x_-, x_+}(\epsilon_1, \epsilon_2 \mid E_i)=\P{}{-x_-\le \epsilon_1, x_+\le \epsilon_2\mid E_i}$. Given $E_i$ happens, $\P{}{-x_-\le  \epsilon_1\mid E_i}=1-(1-\epsilon_1)^{n-i}$ and $\P{}{x_+\le  \epsilon_2\mid E_i}=1-(1-\epsilon_2)^{i}$. Also since $-x_-\le  \epsilon_1$ and $x_+\le  \epsilon_2$ are independent given $E_i$ happens, thus
\begin{equation}
F_{-x_-, x_+}(\epsilon_1, \epsilon_2 \mid E_i)=\P{}{-x_-\le \epsilon_1, x_+\le \epsilon_2\mid E_i}=[1-(1-\epsilon_1)^{n-i}][1-(1-\epsilon_2)^{i}].
\end{equation}
Take the derivative of $F$ gives
\begin{equation}
f_{-x_-, x_+}(\epsilon_1, \epsilon_2 \mid E_i)=i(n-i)(1-\epsilon_1)^{n-i-1}(1-\epsilon_2)^{i-1}.
\end{equation}
Note that $|\hat\theta_S|\le \epsilon\Leftrightarrow |-x_1-x_2|\le 2\epsilon$. Therefore, we integrate $f_{-x_-, x_+}(\epsilon_1, \epsilon_2 \mid E_i)$ over the region $|\epsilon_1-\epsilon_2|\le 2\epsilon$ to obtain $\P{}{|\hat\theta_S|\le\epsilon \mid E_i}$. However, note that $0\le\epsilon_1,\epsilon_2\le1$, thus for $\epsilon> \frac{1}{2}$, the region $|\epsilon_1-\epsilon_2|\le 2\epsilon$ becomes the whole $[0,1]\times[0,1]$ and the integration is 1. Then \eqref{divide} becomes $\P{}{|\hat\theta_S|\le\epsilon, E}=\sum_{i=1}^{n-1}\P{}{E_i}=\P{}{E}=1-(\frac{1}{2})^{n-1}$. For $\epsilon\le\frac{1}{2}$, by \eqref{divide} we have
\begin{equation}\label{inte_sum}
\begin{aligned}
\P{}{|\hat\theta_S|\le\epsilon, E}&=\sum_{i=1}^{n-1}\P{}{E_i}\int_{|\epsilon_1-\epsilon_2|\le2\epsilon}i(n-i)(1-\epsilon_1)^{n-i-1}(1-\epsilon_2)^{i-1}d\epsilon_2d\epsilon_1\\
&=\sum_{i=1}^{n-1}C^n_i(\frac{1}{2})^n\int_{|\epsilon_1-\epsilon_2|\le2\epsilon}i(n-i)(1-\epsilon_1)^{n-i-1}(1-\epsilon_2)^{i-1}d\epsilon_2d\epsilon_1\\
&=(\frac{1}{2})^n\int_{|\epsilon_1-\epsilon_2|\le2\epsilon}\sum_{i=1}^{n-1}C^n_ii(n-i)(1-\epsilon_1)^{n-i-1}(1-\epsilon_2)^{i-1}d\epsilon_2d\epsilon_1.\\
\end{aligned}
\end{equation}
Note that $C^n_ii(n-i)=n(n-1)C^{n-2}_{i-1}$, \eqref{inte_sum} becomes
\begin{equation}\label{inte_sum2}
\begin{aligned}
&\P{}{|\hat\theta_S|\le\epsilon, E}=n(n-1)(\frac{1}{2})^{n}\int_{|\epsilon_1-\epsilon_2|\le2\epsilon}\sum_{i=1}^{n-1}C^{n-2}_{i-1}(1-\epsilon_1)^{n-i-1}(1-\epsilon_2)^{i-1}d\epsilon_2d\epsilon_1\\
&=n(n-1)(\frac{1}{2})^{n}\int_{|\epsilon_1-\epsilon_2|\le2\epsilon}\sum_{i=0}^{n-2}C^{n-2}_i(1-\epsilon_1)^{n-2-i}(1-\epsilon_2)^{i}d\epsilon_2d\epsilon_1\\
&=n(n-1)(\frac{1}{2})^n\int_{|\epsilon_1-\epsilon_2|\le2\epsilon}(2-\epsilon_1-\epsilon_2)^{n-2}d\epsilon_2d\epsilon_1\\
&=n(n-1)(\frac{1}{2})^n[\int_{[0,1]\times[0,1]}(2-\epsilon_1-\epsilon_2)^{n-2}d\epsilon_2d\epsilon_1-\int_{|\epsilon_1-\epsilon_2|>2\epsilon}(2-\epsilon_1-\epsilon_2)^{n-2}d\epsilon_2d\epsilon_1].\\
\end{aligned}
\end{equation}
Now we compute the two integration in \eqref{inte_sum2}
\begin{equation}\label{inte1}
\begin{aligned}
&\int_{[0,1]\times[0,1]}(2-\epsilon_1-\epsilon_2)^{n-2}d\epsilon_2d\epsilon_1=\int\displaylimits_0^1[-\frac{1}{n-1}(2-\epsilon_1-\epsilon_2)^{n-1}| _0^1]d\epsilon_1\\
&=\int\displaylimits_0^1\frac{1}{n-1}[(2-\epsilon_1)^{n-1}-(1-\epsilon_1)^{n-1}]d\epsilon_1=[-\frac{1}{n(n-1)}(2-\epsilon)^n+\frac{1}{n(n-1)}(1-\epsilon_1)^n]\mid_0^1=\frac{2^n-2}{n(n-1)}.
\end{aligned}
\end{equation}
For the second integration, note that it can decomposed as
\begin{equation}
\int_{|\epsilon_1-\epsilon_2|>2\epsilon}(2-\epsilon_1-\epsilon_2)^{n-2}d\epsilon_2d\epsilon_1=\int\displaylimits_0^{1-2\epsilon}\int\displaylimits_{\epsilon_1+2\epsilon}^1(2-\epsilon_1-\epsilon_2)^{n-2}d\epsilon_2d\epsilon_1+\int\displaylimits_{2\epsilon}^1\int\displaylimits_0^{\epsilon_1-2\epsilon}(2-\epsilon_1-\epsilon_2)^{n-2}d\epsilon_2d\epsilon_1.
\end{equation}
Since the two sub-integration's are identical because the two sub regions are symmetric. We only show the computation for the first.
\begin{equation}
\begin{aligned}
&\int\displaylimits_0^{1-2\epsilon}\int\displaylimits_{\epsilon_1+2\epsilon}^1(2-\epsilon_1-\epsilon_2)^{n-2}d\epsilon_2d\epsilon_1=\int\displaylimits_0^{1-2\epsilon}[-\frac{1}{n-1}(2-\epsilon_1-\epsilon_2)^{n-1}| _{\epsilon_1+2\epsilon}^1]d\epsilon_1\\
&=\int\displaylimits_0^{1-2\epsilon}[-\frac{1}{n-1}(1-\epsilon_1)^{n-1}+\frac{2^{n-1}}{n-1}(1-\epsilon_1-\epsilon)^{n-1}]d\epsilon_1\\
&=[\frac{1}{n(n-1)}(1-\epsilon_1)^n-\frac{2^{n-1}}{n(n-1)}(1-\epsilon_1-\epsilon)^n]\mid_0^{1-2\epsilon}\\
&=\frac{2^{n-1}}{n(n-1)}[\epsilon^n+(1-\epsilon)^n]-\frac{1}{n(n-1)}.
\end{aligned}
\end{equation}
Thus we have
\begin{equation}\label{inte2}
\begin{aligned}
&\int_{|\epsilon_1-\epsilon_2|>2\epsilon}(2-\epsilon_1-\epsilon_2)^{n-2}d\epsilon_2d\epsilon_1=2\int\displaylimits_0^{1-2\epsilon}\int\displaylimits_{\epsilon_1+2\epsilon}^1(2-\epsilon_1-\epsilon_2)^{n-2}d\epsilon_2d\epsilon_1\\
&=\frac{2^{n}}{n(n-1)}[\epsilon^n+(1-\epsilon)^n]-\frac{2}{n(n-1)}.
\end{aligned}
\end{equation}
Combine \eqref{inte1} and \eqref{inte2}, we can compute \eqref{inte_sum2} as follows.
\begin{equation}
\begin{aligned}
\P{}{|\hat\theta_S|\le\epsilon, E}=&=n(n-1)(\frac{1}{2})^{n}[\frac{2^n-2}{n(n-1)}-\frac{2^n}{n(n-1)}(\epsilon^n+(1-\epsilon)^n)+\frac{2}{n(n-1)}]\\
&=\frac{2^n-2}{2^n}-\epsilon^n-(1-\epsilon)^n+(\frac{1}{2})^{n-1}\\
&=1-\epsilon^n-(1-\epsilon)^n
\end{aligned}
\end{equation}
Therefore we have
\begin{equation}\label{eq:E}
\P{}{|\hat\theta_S|\le\epsilon, E}=\left\{
\begin{aligned}
&1-\epsilon^n-(1-\epsilon)^n&& \mbox{if $\epsilon\le\frac{1}{2}$}\\
&1-(\frac{1}{2})^{n-1} && \mbox{if $\frac{1}{2}<\epsilon\le1$}.
\end{aligned}
\right.
\end{equation}
Now combine \eqref{eq:Ec} and \eqref{eq:E} we have
\begin{equation}
\P{}{|\hat\theta_S|\le\epsilon}=\left\{
\begin{aligned}
&1-\epsilon^n-(1-\epsilon)^n&& \mbox{if $\epsilon\le\frac{1}{2}$}\\
&1-(\frac{1}{2})^{n-1} && \mbox{if $\frac{1}{2}<\epsilon<1$}\\
&1&&\mbox{if $\epsilon=1$}.
\end{aligned}
\right.
\end{equation}
which is equivalent to \eqref{eq:Amptail}.
\end{proof}

\lemAmshighP*
\begin{proof}
Let $S_1=\{x\mid (x, 1)\in S\}$ and $S_2=\{x\mid (x, -1)\in S\}$ respectively. Then we have $|S_1|+|S_2|=4m$. Define event $E_1:\{|S_1|\ge m\land |S_2|\ge m\}$. Then we have
\begin{equation}
\P{}{E_1}=1-2\sum_{i=0}^{m-1}C^{4m}_i(\frac{1}{2})^{4m}.
\end{equation}
where we rule out all possible sequences of $4m$ points which lead to $|S_1|<m$ or $|S_2|<m$.
By standard result~\cite{Shalev-Shwartz:2014:UML:2621980} (Lemma A.5) $\sum_{k=0}^{d}C^m_k\le(\frac{em}{d})^d$, we have
\begin{equation}
\P{}{E_1}\ge1-2(\frac{4em}{m-1})^{m-1}(\frac{1}{2})^{4m}=1-\frac{1}{2}\frac{e^{m-1}}{4^m}(\frac{m}{m-1})^{m-1}\ge1-\frac{1}{2}(\frac{e}{4})^m\ge1-(\frac{e}{4})^m
\end{equation}
where the 2nd-to-last inequality follows from the fact that $e \ge (1+\frac{1}{m-1})^{m-1}$.
Note that by definition $m\ge\frac{3e}{\ln4-1}\ln\frac{3}{\delta}>\frac{1}{\ln4-1}\ln\frac{3}{\delta}$, thus $(\frac{e}{4})^m<\frac{\delta}{3}$ and $\P{}{E_1}>1-\frac{\delta}{3}$.
Since $|S_1|+|S_2|=4m$, then either $|S_1|\ge2m$ or $|S_2|\ge2m$. Without loss of generality we assume $|S_1|\ge2m$. We then divide the interval [0, 1] equally into $N=\lfloor m^2(\ln\frac{3}{\delta})^{-1} \rfloor$ segments. The length of each segment is $\frac{1}{N}=O(\frac{1}{m^2})$ as Figure~\ref{segments} shows. Note that $m\ge\frac{3e}{\ln4-1}\ln\frac{3}{\delta}>3e\ln\frac{3}{\delta}$, thus $N\ge\lfloor3em\rfloor>2em>m$.
\begin{figure}[H]
\centering
\includegraphics[width=3.5in,height=0.5in]{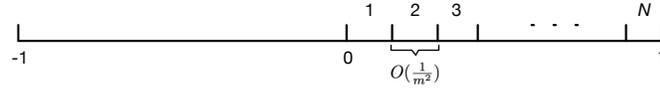}
\caption{segments}
\label{segments}
\end{figure}
Let $N_o$ be the number of segments that are occupied by the points in $S_1$. Note that $N_o$ is a random variable.
Let $E_2$ be the event that $N_o\ge m$. Now we lower bound $\P{}{E_2}$. 
This is a variant of the coupon collector's problem: there are $N$ distinct coupons, and in $|S_1|$ trials we want to collect at least $m$ distinct coupons. 
Note that $\P{}{E_2}=1-\P{}{E_2^c}=1-\sum_{i=1}^{m-1}\P{}{N_o=i}$. 
Let $T_i$ be the number of all possible coupon sequences of $S_1$ such that $S_1$ occupies exactly $i$ segments (i.e. distinct coupons). 
We have $C^n_i$ ways of choosing $i$ segments among a total of $N$. Also, for each choice of $i$ segments, the number of all possible coupon sequences of $S_1$ such that $S_1$ fully occupies those $i$ segments without empty is upper bounded by $i^{|S_1|}$. Thus $T_i\le C^n_i i^{|S_1|}$ and we have
\begin{equation}
\P{}{N_o=i}=\frac{T_i}{N^{|S_1|}}\le C^n_i(\frac{i}{N})^{|S_1|}.
\end{equation}
Since $m\ge\frac{3e}{\ln4-1}\ln\frac{3}{\delta}>\log_2\frac{3}{\delta}$, $|S_1|\ge 2m$, and $N>2em$, thus
\begin{equation}
\begin{aligned}
\P{}{E_2^c}&=\sum_{i=1}^{m-1}\P{}{N_o=i}\le\sum_{i=1}^{m-1}C^n_i(\frac{i}{N})^{|S_1|}\le\sum_{i=1}^{m-1}C^n_i(\frac{m}{N})^{2m}\\
&< \sum_{i=0}^{m}C^n_i(\frac{m}{N})^{2m}
\le (\frac{eN}{m})^m(\frac{m}{N})^{2m}=(\frac{em}{N})^m<(\frac{em}{2em})^m=(\frac{1}{2})^m<\frac{\delta}{3}.
\end{aligned}
\end{equation}
Thus $\P{}{E_2}\ge 1-\frac{\delta}{3}$. Applying union bound, $\P{}{E_1, E_2}\ge 1-\frac{2\delta}{3}$.

Let $E_3$ be the following event: there exist a point $x_2$ in $S_2$ such that $-x_2$, the flipped point, lies in the same segment as some point $x_1$ in $S_1$. If $E_3$ happens, then $|x_1+x_2|=|x_1-(-x_2)|\le\frac{1}{N}$. Note that $\P{}{E_3}\ge \P{}{E_1, E_2, E_3}=\P{}{E_3\mid E_1, E_2}\P{}{E_1, E_2}$. Now we lower bound $\P{}{E_3\mid E_1, E_2}$. Given $E_1$ and $E_2$ happen,  we have $|S_2|\ge m$ and $N_o\ge m$. Since $N=\lfloor m^2(\ln\frac{3}{\delta})^{-1} \rfloor\le m^2(\ln\frac{3}{\delta})^{-1}$, we have
\begin{equation}
\P{}{E_3^c\mid E_1, E_2}=(1-\frac{N_o}{N})^{|S_2|}\le(1-\frac{m}{N})^{|S_2|}\le(1-\frac{m}{N})^{m}\le e^{-\frac{m^2}{N}}\le\frac{\delta}{3}.
\end{equation}
Thus, $\P{}{E_3\mid E_1, E_2}=1-\P{}{E_3^c\mid E_1, E_2}> 1-\frac{\delta}{3}$.
$\P{}{E_3}\ge \P{}{E_1, E_2, E_3}=\P{}{E_3\mid E_1, E_2}\P{}{E_1, E_2}\ge (1-\frac{\delta}{3})(1-\frac{2\delta}{3})> 1-\delta$. Thus with probability at least $1-\delta$, there exist $x_2 \in S_2$ and $x_1 \in S_1$ such that $|x_1+x_2|\le\frac{1}{N}$. 

We now bound $\frac{1}{N}$. 
$N=\lfloor m^2(\ln\frac{3}{\delta})^{-1} \rfloor\ge\frac{1}{2}m^2(\ln\frac{3}{\delta})^{-1}$. Therefore $\frac{1}{N}\le\frac{2}{m^2}\ln\frac{3}{\delta}$. 
Recall by definition $m\ge(\frac{1}{\epsilon}\ln\frac{3}{\delta})^{\frac{1}{2}}$, thus $\frac{1}{N}\le2\epsilon$. 

We now have $|x_1+x_2|\le2\epsilon$. Finally, since $\{s_-, s_+\}$ selected by teacher $B_{ms}$ is the most symmetric pair, it must satisfy $|s_-+s_+|\le|x_1+x_2|\le2\epsilon$.
Putting together, with probability at least $1-\delta$, $R(\hat\theta_{B_{ms}(S)})=\frac{1}{2}|s_-+s_+|\le\epsilon$. 
\end{proof}

\end{document}